%% file: main.tex
\documentclass{article}
\usepackage[T1]{fontenc}
\usepackage[utf8]{inputenc}
\usepackage[american]{babel}

\usepackage[nonatbib,final]{neurips_2018}

\usepackage{amsmath}
\usepackage{amsthm}
\usepackage{url}
\usepackage{booktabs}
\usepackage{amsfonts}
\usepackage{nicefrac}
\usepackage{microtype}
\usepackage{enumitem}
\usepackage{etoolbox}
\usepackage{amstext}
\usepackage{amssymb}
\usepackage{csquotes}
\usepackage{stmaryrd}
\usepackage{xparse}

\usepackage{xcolor}
\usepackage[]{todonotes}

\usepackage{appendix}
\usepackage{graphicx}
\usepackage{epstopdf}
\usepackage{relsize}
\usepackage{subcaption}

\usepackage[colorlinks,linkcolor={red!80!black},citecolor={green!80!black},urlcolor={blue!80!black},hypertexnames=false]{hyperref}

\usepackage[noabbrev,capitalize]{cleveref}
\crefformat{equation}{(#2#1#3)}
\crefrangeformat{equation}{(#3#1#4) to~(#5#2#6)}
\crefmultiformat{equation}{(#2#1#3)}{ and~(#2#1#3)}{, (#2#1#3)}{ and~(#2#1#3)}
\crefname{appsec}{Appendix}{Appendices}

\usepackage[backend=biber,style=numeric,maxcitenames=2,maxbibnames=8,sortcites,isbn=false,doi=false,firstinits=true]{biblatex}
\renewbibmacro{in:}{}

\DeclareMathOperator*{\argmin}{argmin}

\DeclareMathOperator{\diag}{diag}
\newcommand{\distconv}{\xrightarrow{D}}
\newcommand{\wassconv}[1][]{\xrightarrow{{\mathcal W}_{#1}}}
\newcommand{\ud}{\mathrm d}
\newcommand{\dx}{\ud x}
\newcommand{\mudx}{\operatorname{\mu}(\dx)}
\DeclareMathOperator{\D}{\mathcal D}
\DeclareMathOperator{\E}{\mathbb E}

\newcommand{\f}{\mathcal F}
\newcommand{\h}{\mathcal H}
\newcommand{\hs}{\mathrm{HS}}
\DeclareMathOperator{\N}{\mathcal N}

\newcommand{\p}{\mathcal P}
\newcommand{\R}{\mathbb R}
\newcommand{\n}{\mathbb N}
\newcommand{\PP}{\mathbb P}
\newcommand{\QQ}{\mathbb Q}

\DeclareMathOperator{\cond}{cond}
\newcommand{\tp}{^\mathsf{T}}
\newcommand{\x}{\mathcal X}
\newcommand{\y}{\mathcal Y}
\newcommand{\z}{\mathcal Z}
\newcommand{\lip}{\mathrm{Lip}}
\newcommand{\pushforward}{\scalebox{.85}{\#}}
\DeclareMathOperator{\sigmamin}{\sigma_{min}}

\DeclareMathOperator{\MMD}{MMD}
\DeclareMathOperator{\GCMMD}{GCMMD}
\DeclareMathOperator{\SMMD}{SMMD}
\DeclareMathOperator{\LipMMD}{LipMMD}
\newcommand{\optMMD}[1][\Psi]{\operatorname{\mathcal D_{\mathrm{MMD}}^{#1}}}
\NewDocumentCommand{\optLipMMD}{O{\Psi} O{\lambda}}{\operatorname{\mathcal D_{\mathrm{LipMMD}}^{#1,#2}}}
\NewDocumentCommand{\optGCMMD}{O{\Psi} O{\mu} O{\lambda}}{\operatorname{\mathcal D_{\mathrm{GCMMD}}^{#2,#1,#3}}}
\NewDocumentCommand{\optSMMD}{O{\Psi} O{\mu} O{\lambda}}{\operatorname{\mathcal D_{\mathrm{SMMD}}^{#2,#1,#3}}}
\DeclareMathOperator{\DGAN}{\D_{GAN}}

\DeclareMathOperator{\W}{\mathcal{W}}

\newcommand{\httpsurl}[1]{\href{https://#1}{\nolinkurl{#1}}}

\let\citep\parencite
\let\citet\textcite

\makeatletter
\newcommand\given{\@ifstar{\mathrel{}\middle|\mathrel{}}{\mid}}

\DeclareRobustCommand{\abs}{\@ifstar\@abs\@@abs}
\newcommand{\@abs}[1]{\left\lvert #1 \right\rvert}
\newcommand{\@@abs}[1]{\lvert #1 \rvert}

\DeclareRobustCommand{\norm}{\@ifstar\@norm\@@norm}
\newcommand{\@norm}[1]{\left\lVert #1 \right\rVert}
\newcommand{\@@norm}[1]{\lVert #1 \rVert}

\DeclareRobustCommand{\inner}{\@ifstar\@inner\@@inner}
\newcommand{\@inner}[1]{\left\langle #1 \right\rangle}
\newcommand{\@@inner}[1]{\langle #1 \rangle}

\newcommand{\newreptheorem}[2]{%
  \newtheorem*{rep@#1}{\rep@title}%
  \newenvironment{rep#1}[1]{%
    \def\rep@title{#2 \ref*{##1}}\begin{rep@#1}%
  }{%
    \end{rep@#1}%
  }%
}
\makeatother

\newtheorem{lem}{Lemma}
\newtheorem{theorem}{Theorem}
\newreptheorem{theorem}{Theorem}
\newtheorem{prop}[lem]{Proposition}
\newreptheorem{prop}{Proposition}
\newtheorem{example}{Example}

\newlist{assumplist}{enumerate}{1}
\setlist[assumplist]{label=(\textbf{\Alph*})}
\Crefname{assumplisti}{Assumption}{Assumptions}

\newlist{assumplist2}{enumerate}{1}
\setlist[assumplist2]{label=(\textbf{\Roman*})}
\Crefname{assumplist2i}{Assumption}{Assumptions}

\newlist{proplist}{enumerate}{1}
\setlist[proplist]{label=({\roman*})}
\Crefname{proplisti}{Property}{Properties}

\title{On gradient regularizers for MMD GANs}

\author{
  Michael Arbel\thanks{These authors contributed equally.\vspace*{-5mm}}\\
  Gatsby Computational Neuroscience Unit\\University College London\\
  \texttt{michael.n.arbel@gmail.com}
  \And
  Danica J. Sutherland\footnotemark[1]\\
  Gatsby Computational Neuroscience Unit\\University College London\\
  \texttt{djs@djsutherland.ml}
  \And
  Miko{\l}aj Bi\'nkowski\\
  \phantom{xxxxx}Department of Mathematics\phantom{xxxxx}\\Imperial College London\\
  \texttt{mikbinkowski@gmail.com}
  \vspace*{-5mm}
  \And
  Arthur Gretton\\
  Gatsby Computational Neuroscience Unit\\University College London\\
  \texttt{arthur.gretton@gmail.com}
  \vspace*{-5mm}
}

\begin{document}
\maketitle

\begin{abstract}
We propose a principled method for gradient-based regularization of the critic of GAN-like models trained by adversarially optimizing the kernel of a Maximum Mean Discrepancy (MMD).
We show that controlling the gradient of the critic is vital to having a sensible loss function,
and devise a method to enforce exact, analytical gradient constraints
at no additional cost compared to existing approximate techniques based on additive regularizers.
The new loss function is provably continuous,
and experiments show that it stabilizes and accelerates training,
giving image generation models that outperform state-of-the art methods
on $160 \times 160$ CelebA and $64 \times 64$ unconditional ImageNet.
\end{abstract}

\section{Introduction}

There has been an explosion of interest in \emph{implicit generative models} (IGMs) over the last few years,
especially after the introduction of generative adversarial networks (GANs) \parencite{gans}.
These models allow approximate samples from a complex high-dimensional target distribution $\PP$,
using a model distribution $\QQ_\theta$, where estimation of likelihoods, exact inference, and so on are not tractable.
GAN-type IGMs have yielded very impressive empirical results,
particularly for image generation,
far beyond the quality of samples seen from most earlier generative models \parencite[e.g.][]{progressive-growing,dcgan,wgan-gp,munit,anime-gans}.

These excellent results, however, have depended on adding a variety of methods of regularization and other tricks to stabilize the notoriously difficult optimization problem of GANs \parencite{improved-gans,dcgan}.
Some of this difficulty is perhaps because
when a GAN is viewed as minimizing a discrepancy $\DGAN(\PP, \QQ_\theta)$,
its gradient
$\nabla_\theta \DGAN(\PP, \QQ_\theta)$
does not provide useful signal to the generator
if the target and model distributions are not absolutely continuous,
as is nearly always the case \parencite{towards-principled-gans}.

An alternative set of losses are the integral probability metrics (IPMs) \citep{Mueller97},
which can give credit to models $\QQ_{\theta}$  ``near'' to the target
distribution $\PP$ \parencites{wgan}{Bottou:2017}[Section 4 of][]{GneRaf07}.
IPMs are defined in terms of a {\em critic function}: a
``well behaved'' function with  large amplitude
where $\PP$ and $\QQ_\theta$ differ most.
The IPM is  the difference in the expected critic under $\PP$ and $\QQ_\theta$,
and is zero when the distributions agree.
The Wasserstein IPMs, whose critics are made smooth via a Lipschitz constraint,
have been particularly successful in IGMs \citep{wgan,wgan-gp,sinkhorn-igm}.
But the Lipschitz constraint must hold uniformly, which can be hard
to enforce. A popular approximation has been to apply a gradient constraint
only in expectation \citep{wgan-gp}:
the critic's gradient norm is constrained to be small
on points chosen uniformly between $\PP$ and $\QQ$.

Another class of IPMs used as IGM losses
are the Maximum Mean Discrepancies (MMDs) \citep{mmd-jmlr}, as in \citep{gmmn,gen-mmd}. Here the critic
function is a member of a reproducing kernel Hilbert space
(except in \cite{coulomb-gan}, who learn a deep approximation to an RKHS critic).
Better performance can be obtained, however, when the MMD kernel is not based
directly on image pixels, but on learned features
of images. %
Wasserstein-inspired gradient regularization
approaches can be used on the MMD critic when learning these features:
\citep{mmd-gan} uses weight clipping \citep{wgan}, and \citep{Binkowski:2018,cramer-gan}
use a gradient penalty \citep{wgan-gp}.

The recent Sobolev GAN \citep{sobolev-gan}
uses a similar constraint on the expected gradient norm,
but phrases it as estimating a Sobolev IPM rather than loosely approximating Wasserstein.
This expectation can be taken over the same distribution as \citep{wgan-gp},
but other measures are also proposed,
such as $\left(\PP+\QQ_{\theta}\right)/2$.
A second recent approach, the spectrally normalized GAN \citep{Miyato:2018},
controls the Lipschitz constant of the critic by enforcing the spectral norms of the weight matrices to be 1.
Gradient penalties also benefit GANs based on $f$-divergences \citep{NowBotRyo16}:
for instance, the spectral normalization technique of \citep{Miyato:2018} can be applied to the critic network of an $f$-GAN.
Alternatively, a gradient penalty can be defined to approximate the
effect of blurring $\PP$ and $\QQ_\theta$ with noise \citep{roth:regularization}, which
addresses the problem of non-overlapping support \parencite{towards-principled-gans}.
This approach has recently been shown to yield locally convergent optimization
in some cases with non-continuous distributions,
where the original GAN does not \parencite{Mescheder:2018}.

In this paper, we introduce a novel regularization
for the MMD GAN critic of \citep{cramer-gan,mmd-gan,Binkowski:2018},
which {\em directly targets generator performance},
rather than
adopting regularization methods intended to approximate Wasserstein distances \cite{wgan,wgan-gp}.
The new MMD regularizer derives from an approach widely used in semi-supervised
learning \parencite[][Section 2]{Bousquet:2004}, where the aim is to
define a classification function $f$ which is positive on $\PP$
(the positive class) and negative on $\QQ_{\theta}$ (negative class),
in the absence of labels on many of the samples.
The decision boundary between the classes is assumed to be in a region
of low density for both $\PP$ and $\QQ_{\theta}$: $f$ should therefore
be flat where $\PP$ and $\QQ_{\theta}$ have support (areas with
constant label), and have a larger slope in regions of low density.
\Textcite{Bousquet:2004} propose as their regularizer on $f$ a sum of the
variance and a density-weighted gradient norm.

We adopt a related penalty on the MMD critic,
with the difference that we only apply the penalty on $\PP$:
thus, the critic is flatter where $\PP$ has high mass,
but does not vanish on the generator samples from $\QQ_{\theta}$ (which we optimize).
In excluding $\QQ_{\theta}$ from the critic function constraint,
we also avoid the concern raised by \citep{Miyato:2018}
that a critic depending on $\QQ_{\theta}$ will change with the current minibatch --
potentially leading to less stable learning.
The resulting discrepancy is no longer an integral probability metric:
it is asymmetric,
and the critic function class depends on the target $\PP$ being approximated.

We first discuss in \cref{sec:igm-losses} how MMD-based losses can be used to learn implicit generative models, and how a naive approach could fail.
This motivates our new discrepancies, introduced in \cref{sec:new_discrepancies}.
\Cref{sec:experiments} demonstrates that these losses outperform state-of-the-art models for image generation.

\section{Learning implicit generative models with MMD-based losses} \label{sec:igm-losses}
An IGM is a model $\QQ_\theta$ which aims to approximate a target distribution $\PP$
over a space $\x \subseteq \R^d$.
We will define $\QQ_\theta$ by a \emph{generator} function $G_\theta : \z \to \x$,
implemented as a deep network with parameters $\theta$,
where $\z$ is a space of latent codes, say $\R^{128}$.
We assume a fixed distribution on $\z$,
say $Z \sim \mathrm{Uniform}\left( [-1, 1]^{128} \right)$,
and call $\QQ_\theta$ the distribution of $G_\theta(Z)$.
We will consider learning
by minimizing a
discrepancy $\D$ between distributions,
with $\D(\PP, \QQ_\theta ) \ge 0$ and $\D(\PP, \PP) = 0$,
which we call our {\em loss}.
We aim to minimize $\D(\PP, \QQ_\theta)$ with
stochastic gradient descent on an estimator of $\D$.%

In the present work,
we will build losses $\D$
based on the Maximum Mean Discrepancy,
\begin{equation}
  \MMD_k(\PP, \QQ)
  = \sup_{f \,:\, \lVert f \rVert_{\h_k} \le 1}
    \E_{X \sim \PP}[ f(X) ] - \E_{Y \sim \QQ}[ f(Y) ]
  \label{eq:mmd}
,\end{equation}
an integral probability metric where the critic class is the unit ball within $\h_k$,
the reproducing kernel Hilbert space with a kernel $k$.
The optimization in \eqref{eq:mmd} admits a simple closed-form optimal critic,
$f^*(t) \propto \E_{X \sim \PP}[ k(X, t) ] - \E_{Y \sim \QQ}[ k(Y, t) ]$.
There is also an unbiased, closed-form estimator of $\MMD_k^2$
with appealing statistical properties \parencite{mmd-jmlr}~--~in particular, its sample complexity is \emph{independent} of the dimension of $\x$,
compared to the exponential dependence \parencite{weed:wasserstein-rates} of the Wasserstein distance
\begin{equation}
  \W(\PP, \QQ)
  = \sup_{f \,:\, \norm{f}_\lip \le 1}
    \E_{X \sim \PP}[ f(X) ] - \E_{Y \sim \QQ}[ f(Y) ]
  \label{eq:wasserstein}
.\end{equation}

The MMD is \emph{continuous in the weak topology}
for any bounded kernel with Lipschitz embeddings \parencite[Theorem 3.2(b)]{opt-est-probabilities},
meaning that if $\PP_n$ converges in distribution to $\PP$, $\PP_n \distconv \PP$, then $\MMD(\PP_n, \PP) \to 0$.
($\W$ is continuous in the slightly stronger Wasserstein topology \citep[Definition 6.9]{Villani:2009};
$\PP_n \wassconv \PP$ implies $\PP_n \distconv \PP$,
and the two notions coincide if $\mathcal X$ is bounded.)
Continuity means the loss can provide better signal to the generator as $\QQ_\theta$ approaches $\PP$,
as opposed to e.g.\ Jensen-Shannon where the loss could be constant until suddenly jumping to 0 \parencite[e.g.][Example 1]{wgan}.
The MMD is also {\em strict}, meaning it is zero iff $\PP=\QQ_\theta$, for \emph{characteristic} kernels \parencite{SriFukLan11}.
The Gaussian kernel yields an MMD both continuous in the weak topology and strict.
Thus in principle, one need not conduct any alternating optimization in an IGM at all,
but merely choose generator parameters $\theta$ to minimize $\MMD_k$.

Despite these appealing properties,
using simple pixel-level kernels leads to poor generator samples \parencite{gen-mmd,gmmn,opt-mmd,Bottou:2017}.
More recent MMD GANs \parencite{mmd-gan,cramer-gan,Binkowski:2018}
achieve better results by using a parameterized \emph{family} of kernels, $\{ k_\psi \}_{\psi \in \Psi}$,
in the Optimized MMD loss
previously studied by \cite{kernel-choice-mmd,opt-est-probabilities}:
\begin{equation}
  \optMMD[\Psi](\PP, \QQ) := \sup_{\psi \in \Psi} \MMD_{k_{\psi}}(\PP, \QQ)
  \label{eq:Optimized_MMD}
.\end{equation}

We primarily consider kernels defined by some fixed kernel $K$
on top of a learned low-dimensional representation $\phi_\psi : \x \to \R^s$,
i.e.\ $k_\psi(x, y) = K(\phi_\psi(x), \phi_\psi(y))$,
denoted $k_\psi = K \circ \phi_\psi$.
In practice,
$K$ is a simple characteristic kernel, e.g.\ Gaussian,
and $\phi_\psi$ is usually a deep network
with output dimension say $s = 16$ \citep{Binkowski:2018} or even $s = 1$ (in our experiments).
If $\phi_\psi$ is powerful enough, this choice is sufficient;
we need not try to ensure each $k_\psi$ is characteristic, as did \cite{mmd-gan}.
\begin{prop} \label{prop:optmmd:strict}
  Suppose $k = K \circ \phi_\psi$,
  with $K$ characteristic
  and $\{\phi_\psi\}$ rich enough that
  for any $\PP \ne \QQ$,
  there is a $\psi \in \Psi$ for which $\phi_\psi \pushforward \PP \ne \phi_\psi \pushforward \QQ$.\footnote{
    $f \pushforward \PP$ denotes the \emph{pushforward} of a distribution:
    if $X \sim \PP$, then $f(X) \sim f \pushforward \PP$.}
  Then %
  if $\PP \ne \QQ$, $\optMMD(\PP, \QQ) > 0$.
\end{prop}
\vspace*{-3mm}
\begin{proof}
  Let $\hat\psi \in \Psi$ be such that $\phi_{\hat\psi}(\PP) \ne \phi_{\hat\psi}(\QQ)$.
  Then, since $K$ is characteristic,
  \[
    \optMMD(\PP, \QQ)
    = \sup_{\psi \in \Psi} \MMD_K(\phi_\psi \pushforward \PP, \phi_\psi \pushforward \QQ)
    \ge \MMD_K(\phi_{\hat\psi} \pushforward \PP, \phi_{\hat\psi} \pushforward \QQ)
    > 0
  . \qedhere \]
\end{proof}
\vspace*{-3mm}
To estimate $\optMMD$,
one can conduct alternating optimization to estimate a $\hat\psi$
and then update the generator according to $\MMD_{k_{\hat\psi}}$,
similar to the scheme used in GANs and WGANs.
(This form of estimator is justified by an envelope theorem \citep{envelope-thm},
although it is invariably biased \parencite{Binkowski:2018}.)
Unlike $\DGAN$ or $\W$, %
fixing a $\hat\psi$ and optimizing the generator still yields a sensible distance $\MMD_{k_{\hat\psi}}$.

Early attempts at minimizing $\optMMD$ in an IGM, though, were unsuccessful \parencite[footnote 7]{opt-mmd}.
This could be because for some kernel classes,
$\optMMD$ is stronger than Wasserstein or MMD.

\begin{figure}[t]
\centering
        \includegraphics[width=1.\linewidth]{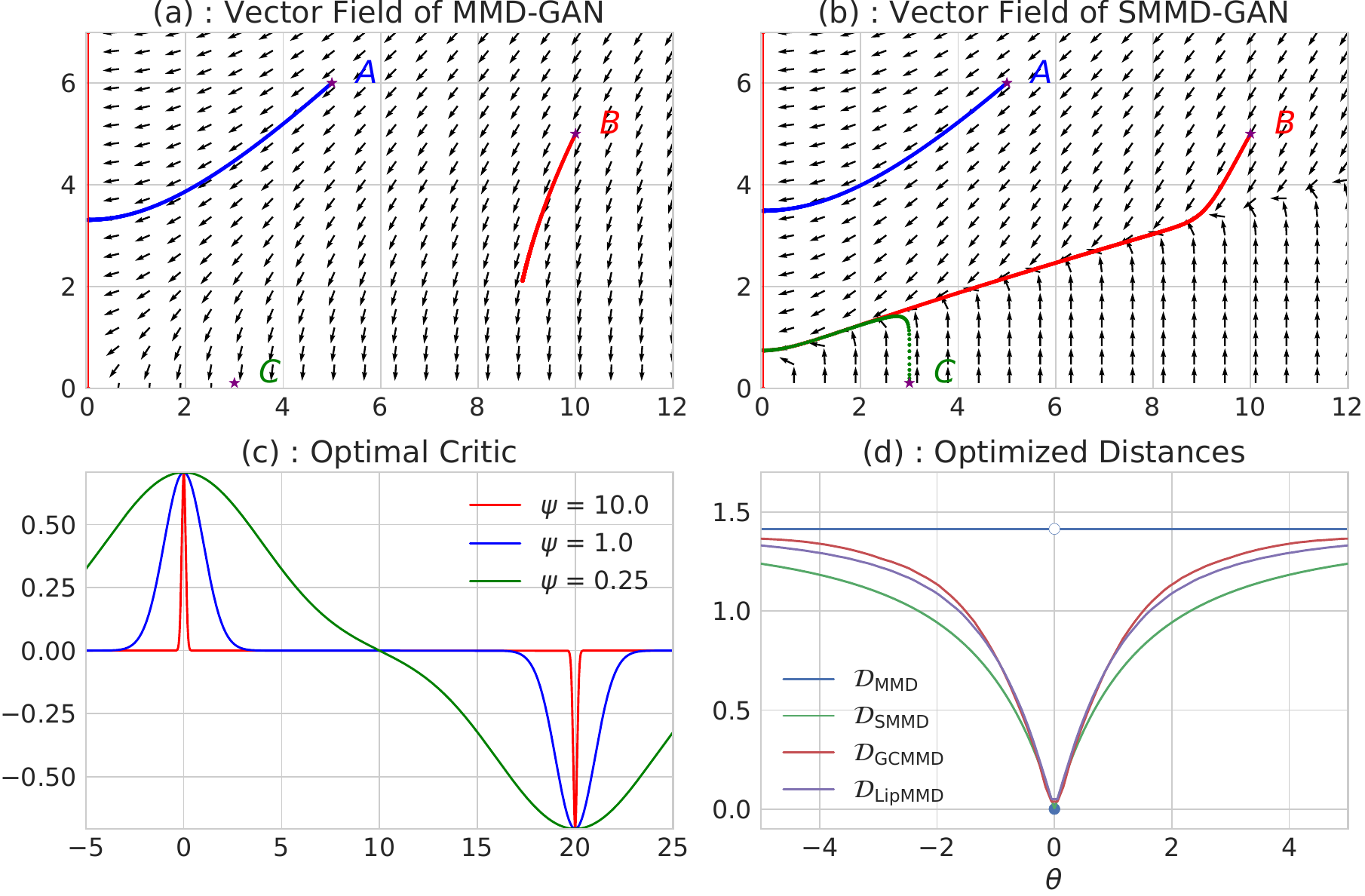}
        \caption{The setting of \cref{example:diracgan}.
         (a, b): parameter-space gradient fields for the MMD and the SMMD (\cref{subsec:Scaled-MMD}); the horizontal axis is $\theta$, and the vertical $1/\psi$.
         (c): optimal MMD critics for $\theta = 20$ with different kernels.
         (d): the MMD and the distances of \cref{sec:new_discrepancies} optimized over $\psi$.}
       \label{fig:mmd_vector_fields}
\end{figure}
\begin{example}[DiracGAN \citep{Mescheder:2018}] \label{example:diracgan}
We wish to model a point mass at the origin of $\R$,
$\PP = \delta_0$,
with any possible point mass, $\QQ_\theta = \delta_\theta$ for $\theta \in \R$.
We use a Gaussian kernel of any bandwidth,
which can be written as $k_\psi = K \circ \phi_\psi$
with $\phi_\psi(x) = \psi x$ for $\psi \in \Psi = \R$
and
$K(a, b) = \exp\left( -\frac12 (a-b)^2 \right)$.
Then
\begin{equation*}
  \MMD_{k_{\psi}}^{2}(\delta_{0}, \delta_{\theta}) = 2 \left( 1 - \exp\left(-\tfrac12 \psi^2 \theta^2 \right) \right),
  \qquad
  \optMMD(\delta_0, \delta_\theta) = \begin{cases} \sqrt{2} & \theta \ne 0 \\ 0 & \theta = 0 \end{cases}
.\end{equation*}
\end{example}
Considering $\optMMD(\delta_0, \delta_{1/n}) = \sqrt 2 \not\to 0$,
even though $\delta_{1/n} \wassconv \delta_0$,
shows that the Optimized MMD
distance is not continuous in the weak or Wasserstein topologies.

This also causes optimization issues.
\Cref{fig:mmd_vector_fields} (a) shows gradient vector fields in parameter space,
$
  v(\theta, \psi) \propto \big(
    -\nabla_\theta \MMD_{k_\psi}^2(\delta_0, \delta_\theta),
     \nabla_\psi   \MMD_{k_\psi}^2(\delta_0, \delta_\theta)
  \big)
$.
Some sequences following $v$ (e.g.\ A) converge to an optimal solution $(0, \psi)$, %
but some (B) move in the wrong direction,
and others (C) are stuck because there is essentially no gradient.
\Cref{fig:mmd_vector_fields} (c, red) shows that %
the optimal $\optMMD$ critic is very sharp near $\PP$ and $\QQ$;
this is less true for cases where the algorithm converged.

We can avoid these issues if we ensure a bounded Lipschitz critic:\footnote{\cite[Theorem 4]{mmd-gan} makes a similar claim to \cref{prop:optmmd:weakness}, but its proof was incorrect: it tries to uniformly bound $\MMD_{k_\psi} \le \W^2$, but the bound used is for a Wasserstein in terms of $\lVert k_\psi(x, \cdot) - k_\psi(y, \cdot) \rVert_{\h_{k_\psi}}$.}
\begin{prop} \label{prop:optmmd:weakness}
  Assume the critics
  $
    f_\psi(x) = (\E_{X \sim \PP} k_\psi(X, x) - \E_{Y \sim \QQ} k_\psi(Y, x)) / \MMD_{k_\psi}(\PP, \QQ)
  $
  are uniformly bounded and have a common Lipschitz constant:
  $\sup_{x \in \x, \psi \in \Psi} \lvert f_\psi(x) \rvert < \infty$ and
  $\sup_{\psi \in \Psi} \lVert f_\psi \rVert_\lip < \infty$.
  In particular, this holds when $k_\psi = K \circ \phi_\psi$ and
  \[
    \sup_{a \in \R^s} K(a, a) < \infty
    ,\quad
    \lVert K(a, \cdot) - K(b, \cdot) \rVert_{\h_K} \le L_K \lVert a - b \rVert_{\R^s}
    ,\quad
    \sup_{\psi \in \Psi} \lVert \phi_\psi \rVert_\lip \le L_\phi < \infty
    \label{eq:stack-kernel-weakness}
  .\]
  Then $\optMMD$ is continuous in the weak topology:
  if $\PP_n \distconv \PP$, then $\optMMD(\PP_n, \PP) \to 0$.
\end{prop}
\begin{proof}
  The main result is \cite[Corollary 11.3.4]{dudley:analysis}.
  To show the claim for $k_\psi = K \circ \phi_\psi$, note that
  $\lvert f_\psi(x) - f_\psi(y) \rvert
  \le \lVert f_\psi \rVert_{\h_{k_\psi}} \lVert k_\psi(x, \cdot) - k_\psi(y, \cdot) \rVert_{\h_{k_\psi}}$,
  which since $\lVert f_\psi \rVert_{\h_{k_\psi}} = 1$ is
  \[
     \lVert K(\phi_\psi(x), \cdot) - K(\phi_\psi(y), \cdot) \rVert_{\h_K}
     \le L_K \lVert \phi_\psi(x) - \phi_\psi(y) \rVert_{\R^s}
     \le L_K L_\phi \lVert x - y \rVert_{\R^d}
  . \qedhere\]
\end{proof}

Indeed, if we put a box constraint on $\psi$ \parencite{mmd-gan}
or regularize the gradient of the critic function \parencite{Binkowski:2018},
the resulting MMD GAN generally matches or outperforms WGAN-based models.
Unfortunately, though,
an additive gradient penalty doesn't substantially change the vector field of \cref{fig:mmd_vector_fields} (a),
as shown in \cref{fig:vector_fields_all} (\cref{appendix:diracgan-full}).
We will propose distances with much better convergence behavior.

\section{New discrepancies for learning implicit generative models}\label{sec:new_discrepancies}

Our aim here is to introduce a discrepancy
that can provide useful gradient information when used as an IGM loss.
Proofs of results in this section are deferred to \cref{appendix:proofs}.

\subsection{Lipschitz Maximum Mean Discrepancy} \label{sec:lipmmd}
\Cref{prop:optmmd:weakness} shows that an MMD-like discrepancy
can be continuous under the weak topology
even when optimizing over kernels,
if we directly restrict the critic functions to be Lipschitz.
We can easily define such a distance, which we call the Lipschitz MMD: for some $\lambda > 0$,
\begin{align}
  \LipMMD_{k,\lambda}(\PP, \QQ) := \sup_{f \in \h_k \,:\, \norm{f}^2_\lip + \lambda \norm{f}^2_{\h_k} \leq1}
     \E_{X \sim \PP}\left[f(X)\right]-\E_{Y \sim \QQ}\left[f(Y)\right]
  \label{eq:LipMMD}
.\end{align}
For a universal kernel $k$, we conjecture that $\lim_{\lambda \to 0} \LipMMD_{k,\lambda}(\PP, \QQ) \to \W(\PP, \QQ)$.
But for any $k$ and $\lambda$,
$\LipMMD$ is upper-bounded by $\W$,
as \eqref{eq:LipMMD} optimizes over a smaller set of functions than \eqref{eq:wasserstein}.
Thus
$ %
  \optLipMMD(\PP, \QQ) := \sup_{\psi\in \Psi} \LipMMD_{k_\psi,\lambda}(\PP, \QQ)
$ %
is also upper-bounded by $\W$,
and hence is continuous in the Wasserstein topology.
It also shows excellent empirical behavior on \cref{example:diracgan}
(\cref{fig:mmd_vector_fields} (d), and \cref{fig:vector_fields_all} in \cref{appendix:diracgan-full}).
But estimating $\LipMMD_{k,\lambda}$, let alone $\D_{\LipMMD}^{\Psi,\lambda}$,
is in general extremely difficult (\cref{sec:est-lipmmd}),
as finding $\norm{f}_\lip$ requires optimization in the input space.
Constraining the \emph{mean} gradient rather than the \emph{maximum},
as we will do next,
is far more tractable.

\subsection{Gradient-Constrained Maximum Mean Discrepancy}
We define the Gradient-Constrained MMD
for $\lambda > 0$
and using some measure $\mu$ as
\begin{align}
  \GCMMD_{\mu,k,\lambda}(\PP,\QQ)
  &:= \sup_{f \in \h_k \,:\, \norm{f}_{S(\mu),k,\lambda} \le 1}
     \E_{X \sim \PP}\left[f(X)\right] - \E_{Y \sim \QQ}\left[f(Y)\right]
  ,\label{eq:SobolevMMD}
\\
  \text{where }
  \lVert f \rVert_{S(\mu),k,\lambda}^{2}
  &:= \lVert f \rVert_{L^{2}(\mu)}^{2}
    + \lVert \nabla f \rVert_{L^{2}(\mu)}^{2}
    + \lambda \lVert f \rVert_{\h_k}^{2}
  \label{eq:sobolev-norm}
.\end{align}
$\norm{\cdot}_{L^{2}(\mu)}^2 = \int \norm{\cdot}^2 \mudx$ denotes the squared $L^2$ norm.
Rather than directly constraining the Lipschitz constant,
the second term $\lVert \nabla f \rVert_{L^{2}(\mu)}^{2}$
encourages the function $f$ to be flat where $\mu$ has mass.
In experiments we use $\mu = \PP$,
flattening the critic near the target sample.
We add the first term following \cite{Bousquet:2004}:
in one dimension and with $\mu$ uniform,
$\lVert \cdot \rVert_{S(\mu),\cdot,0}$
is then an RKHS norm with the kernel
$\kappa(x,y)=\exp (-\| x-y \|)$,
which is also a Sobolev space.
The correspondence to a Sobolev norm is lost in higher dimensions \citep[][Ch. 10]{Wendland05},
but we also found the first term to be beneficial in practice.

We can exploit some properties of $\h_k$
to compute \eqref{eq:SobolevMMD} analytically.
Call the difference in kernel mean embeddings
$\eta := \E_{X \sim \PP}[ k(X, \cdot) ] - \E_{Y \sim \QQ}[ k(Y, \cdot) ] \in \h_k$;
recall $\MMD(\PP, \QQ) = \lVert \eta \rVert_{\h_k}$.
\begin{prop} \label{prop:Finite_rank_approx}
  Let $\hat\mu = \sum_{m=1}^M \delta_{X_m}$. %
  Define $\eta(X) \in \R^M$ with $m$th entry $\eta(X_m)$,
  and $\nabla \eta(X) \in \R^{M d}$ with $(m, i)$th entry\footnote{%
    We use $(m, i)$ to denote $(m-1) d + i$;
    thus $\nabla \eta(X)$ stacks $\nabla \eta(X_1)$, \dots, $\nabla \eta(X_M)$ into one vector.
  } $\partial_i \eta(X_m)$.
  Then under \cref{Moments,Growth,Differentiability,Integrability} in \cref{sec:proofs:distances},
  \begin{align*}
    \GCMMD_{\hat\mu,k,\lambda}^2(\PP, \QQ)
    &= \frac1\lambda \left( \MMD^{2}(\PP, \QQ) - \bar{P}(\eta) \right)
  \\
    \bar{P}(\eta)
    &=
    \begin{bmatrix} \eta(X) \\ \nabla\eta(X) \end{bmatrix}\tp
    \left(
      \begin{bmatrix}
        K & G\tp \\
        G & H
      \end{bmatrix}
      + M \lambda I_{M + M d}
    \right)^{-1}
    \begin{bmatrix} \eta(X) \\ \nabla\eta(X) \end{bmatrix}
  ,\end{align*}
  where $K$ is the kernel matrix $K_{m,m'} = k(X_m, X_{m'})$,
  $G$ is the matrix of left derivatives \footnote{%
    We use $\partial_i k(x, y)$ to denote the partial derivative with respect to $x_i$,
    and $\partial_{i+d} k(x, y)$ that for $y_i$.
  }
  $G_{(m, i), m'} = \partial_i k(X_m, X_{m'})$,
  and $H$ that of derivatives of both arguments $H_{(m, i), (m', j)} = \partial_i \partial_{j+d} k(X_m, X_{m'})$.
\end{prop}
As long as $\PP$ and $\QQ$ have integrable first moments, and $\mu$ has second moments,
\cref{Moments,Growth,Differentiability,Integrability} are satisfied e.g.\ by a Gaussian or linear kernel on top of a differentiable $\phi_\psi$.
We can thus estimate the GCMMD based on samples from $\PP$, $\QQ$, and $\mu$ by using the empirical mean $\hat\eta$ for $\eta$.

This discrepancy indeed works well in practice:
\cref{appendix:sobolev-expt} shows that optimizing our estimate of
$\optGCMMD = \sup_{\psi \in \Psi} \GCMMD_{\mu,k_\psi,\lambda}$
yields a good generative model on MNIST.
But the linear system of size $M+Md$ is impractical:
even on $28 \times 28$ images
and using a low-rank approximation,
the model took days to converge.
We therefore design a less expensive discrepancy in the next section.

The GCMMD is related to some discrepancies previously used in IGM training.
The Fisher GAN \citep{fisher-gan}
uses only the variance constraint
$\lVert f\rVert_{L^{2}(\mu)}^{2} \le 1$.
The Sobolev GAN \citep{sobolev-gan}
constrains $\lVert\nabla f\rVert_{L^{2}(\mu)}^{2}\le1$,
along with a vanishing boundary condition on $f$
to ensure a well-defined solution (although this was not used in the implementation,
and can cause very unintuitive critic behavior; see \cref{appendix:critic-vector-fields}).
The authors considered several choices of $\mu$,
including the WGAN-GP  measure \citep{wgan-gp} and mixtures $\left(\PP+\QQ_{\theta}\right)/2$.
Rather than enforcing the constraints in closed form as we do, though,
these models used additive regularization. %
We will compare to the Sobolev GAN in experiments.

\subsection{Scaled Maximum Mean Discrepancy \label{subsec:Scaled-MMD}}

We will now derive a lower bound on the Gradient-Constrained MMD
which retains many of its attractive qualities
but can be estimated in time linear in the dimension $d$.
\begin{prop} \label{prop:Sobolev_Upperbound}
  Make \cref{Moments,Growth,Differentiability,Integrability}.
  For any $f \in \h_k$,
  $
    \lVert f \rVert_{S(\mu),k,\lambda}
    \le \sigma_{\mu,k,\lambda}^{-1} \lVert f \rVert_{\h_k}
  $,
  where
  \[
  \sigma_{\mu,k,\lambda} := 1 \Big/ \sqrt{
    \lambda + \int k(x, x) \mu(\ud x)
    + \sum_{i=1}^d \int \frac{\partial^2 k(y, z)}{\partial y_i \partial z_i} \Big\rvert_{(y,z) = (x,x)} \mudx
  }
  .\]
\end{prop}
We then define the Scaled Maximum Mean Discrepancy based on this bound of \cref{prop:Sobolev_Upperbound}:
\begin{equation}
  \SMMD_{\mu,k,\lambda}(\PP,\QQ)
  := \sup_{f \,:\, \sigma_{\mu,k,\lambda}^{-1}\Vert f\Vert_{\h}\leq1 }
     \E_{X \sim \PP}\left[ f(X) \right] - \E_{Y \sim \QQ}\left[ f(Y) \right]
   = \sigma_{\mu,k,\lambda} \MMD_k(\PP, \QQ)
  \label{eq:Scaled_MMD}
.\end{equation}
Because the constraint in the optimization of \eqref{eq:Scaled_MMD}
is more restrictive than in that of \eqref{eq:SobolevMMD},
we have that $\SMMD_{\mu,k,\lambda}(\PP, \QQ) \le \GCMMD_{\mu,k,\lambda}(\PP, \QQ)$.
The Sobolev norm $\norm{f}_{S(\mu),\lambda}$,
and a fortiori the gradient norm under $\mu$,
is thus also controlled for the SMMD critic.
We also show in
\cref{appendix:gcmmd-smmd}
that $\SMMD_{\mu,k,\lambda}$ behaves similarly to $\GCMMD_{\mu,k,\lambda}$
on Gaussians.

If $k_\psi = K \circ \phi_\psi$
    and %
    $K(a, b) = g(-\lVert a - b \rVert^{2})$,
    then $\sigma_{k,\mu,\lambda}^{-2} = \lambda + g(0) + 2 \lvert g'(0) \rvert \E_{\mu}\left[ \lVert \nabla\phi_\psi(X) \rVert_F^2 \right]$.
    Or if $K$ is linear, $K(a, b) = a\tp b$,
    then $\sigma_{k,\mu,\lambda}^{-2} = \lambda + \E_{\mu}\left[ \lVert \phi_\psi(X) \rVert^{2} + \lVert \nabla\phi_\psi(X) \rVert_F^2 \right]$.
Estimating these terms based on samples from $\mu$ is straightforward,
giving a natural estimator for the SMMD.

Of course,
if $\mu$ and $k$ are fixed,
the SMMD is simply a constant times the MMD,
and so behaves in essentially the same way as the MMD.
But optimizing the SMMD over a kernel family $\Psi$,
$\optSMMD(\PP,\QQ) := \sup_{\psi\in\Psi} \SMMD_{\mu,k_\psi,\lambda}(\PP, \QQ)$,
gives a distance very different from $\optMMD$ \eqref{eq:Optimized_MMD}.

\Cref{fig:mmd_vector_fields} (b) shows the vector field for the Optimized SMMD loss in \cref{example:diracgan},
using the WGAN-GP measure $\mu = \operatorname{Uniform}(0, \theta)$.
The optimization surface is far more amenable:
in particular the location $C$,
which formerly had an extremely small gradient that made learning effectively impossible,
now converges very quickly by first reducing the critic gradient until some signal is available.
\Cref{fig:mmd_vector_fields} (d) demonstrates that
$\optSMMD$, like $\optGCMMD$ and $\optLipMMD$
but in sharp contrast to $\optMMD$,
is continuous with respect to the location $\theta$ and provides a strong gradient towards 0.

We can establish that $\optSMMD$ is continuous in the Wasserstein topology under some conditions:
\begin{theorem}\label{thm:continuity_opt_SMMD}
Let $k_\psi = K\circ \phi_{\psi}$,
with $\phi_{\psi} : \mathcal X \to \R^s$ a fully-connected $L$-layer network
with Leaky-ReLU$_\alpha$ activations
whose layers do not increase in width,
and $K$ satisfying mild smoothness conditions $Q_K < \infty$
(\cref{decreasing_dimensions,leaky_relu,Lichitz_kernel,Convexe_Hessian} in \cref{appendix:continuity_opt_smmd}).
Let $\Psi^\kappa$ be the set of parameters where each layer's weight matrices have condition number $\cond(W^l) = \norm{W^l} / \sigmamin(W^l) \le \kappa < \infty$.
If $\mu$ has a density (\cref{full_support}),
then
\[
  \optSMMD[\Psi^\kappa](\PP, \QQ) \le \frac{Q_K \kappa^{L/2}}{\sqrt{d_L} \alpha^{L/2}} \W(\PP,\QQ)
.\]
Thus if $\PP_n \wassconv \PP$, then $\optSMMD[\Psi^\kappa](\PP_n, \PP) \to 0$,
even if $\mu$ is chosen to depend on $\PP$ and $\QQ$.
\end{theorem}

\paragraph{Uniform bounds vs bounds in expectation}
Controlling $\lVert \nabla f_\psi \rVert_{L^2(\mu)}^2 = \E_\mu \lVert \nabla f_\psi(X) \rVert^2$
does not necessarily imply a bound on $\lVert f \rVert_\lip \ge \sup_{x \in \x} \lVert \nabla f_\psi(X) \rVert$,
and so does not in general give continuity via \cref{prop:optmmd:weakness}.
\cref{thm:continuity_opt_SMMD} implies that when the network's weights are well-conditioned,
it is sufficient to only control $\lVert \nabla f_\psi \rVert_{L^2(\mu)}^2$,
which is far easier in practice than controlling $\Vert f\Vert_\lip$.
If we instead tried to directly controlled $\lVert f \rVert_\lip$ with e.g.\ spectral normalization (SN) \citep{Miyato:2018},
we could significantly reduce the expressiveness of the parametric family.
In \cref{example:diracgan},
constraining $\lVert \phi_\psi \rVert_\lip = 1$ limits us to only $\Psi = \{ 1 \}$.
Thus $\optMMD[\{1\}]$ is simply the $\MMD$ with an RBF kernel of bandwidth 1,
which has poor gradients when $\theta$ is far from $0$ (\cref{fig:mmd_vector_fields} (c), blue).
The Cauchy-Schwartz bound of \cref{prop:Sobolev_Upperbound}
allows jointly adjusting the smoothness of $k_\psi$ and the critic $f$,
while SN must control the two independently.
Relatedly, limiting $\lVert \phi \rVert_\lip$ by limiting the Lipschitz norm of each layer
could substantially reduce capacity,
while $\lVert \nabla f_\psi \rVert_{L^2(\mu)}$ need not be decomposed by layer.
Another advantage is that $\mu$ %
provides a data-dependent measure of complexity as in \cite{Bousquet:2004}:
we do not needlessly prevent ourselves from using critics that behave poorly only far from the data.

\paragraph{Spectral parametrization}\label{par:parametrization}
When the generator is near a local optimum,
the critic might identify only one direction on which $\QQ_\theta$ and $\PP$ differ.
If the generator parameterization is such that there is no local way for the generator to correct it,
the critic may begin to single-mindedly focus on this difference,
choosing redundant convolutional filters
and causing the condition number of the weights to diverge.
If this occurs, the generator will be motivated to fix this single direction
while ignoring all other aspects of the distributions,
after which it may become stuck.
We can help avoid this collapse by using a critic parameterization
that encourages diverse filters with higher-rank weight matrices.
\Textcite{Miyato:2018} propose to parameterize the weight matrices as
$W = \gamma \bar{W} / \lVert \bar{W} \rVert_\text{op}$,
where $\Vert \bar{W} \Vert_\text{op}$ is the spectral norm of $\bar{W}$.
This parametrization works particularly well with $\optSMMD$;
\cref{fig:celebA_scores_and_singular_values} (b) shows the singular values of the second layer of a critic's network (and \cref{fig:singular_values-full}, in \cref{appendix:smmd_vs_sn}, shows more layers), while \cref{fig:celebA_scores_and_singular_values} (d) shows the evolution of the condition number during training.
The conditioning of the weight matrix remains stable throughout training with spectral parametrization,
while it worsens through training in the default case.

\section{Experiments} \label{sec:experiments}

We evaluated unsupervised image generation on three datasets:
CIFAR-10 \parencite{cifar10} ($60\,000$ images, $32\times32$),
CelebA \parencite{celeba} ($202\,599$ face images, resized and cropped to $160\times160$ as in \cite{Binkowski:2018}),
and the more challenging ILSVRC2012 (ImageNet) dataset \parencite{Russakovsky:2014} ($1\,281\,167$ images, resized to $64\times64$).
Code for all of these experiments is available at
\httpsurl{github.com/MichaelArbel/Scaled-MMD-GAN}.

\textbf{Losses} %
All models are based on a scalar-output critic network $\phi_\psi : \x \to \R$,
except MMDGAN-GP where $\phi_\psi : \x \to \R^{16}$ as in \cite{Binkowski:2018}.
The WGAN and Sobolev GAN use a critic $f = \phi_\psi$,
while the GAN uses a discriminator $D_\psi(x) = 1 / (1 + \exp(-\phi_\psi(x)))$.
The MMD-based methods use a kernel $k_\psi(x, y) = \exp( - (\phi_\psi(x) - \phi_\psi(y))^2 / 2 )$,
except for MMDGAN-GP which uses a mixture of RQ kernels as in \cite{Binkowski:2018}.
Increasing the output dimension of the critic or using a different kernel didn't substantially change the performance of our proposed method.
We also consider SMMD with a linear top-level kernel, $k(x, y) = \phi_\psi(x) \phi_\psi(y)$;
because this becomes essentially identical to a WGAN (\cref{appendix:wgan-linear-kernel}),
we refer to this method as SWGAN.
SMMD and SWGAN use $\mu = \PP$; Sobolev GAN uses $\mu = (\PP + \QQ) / 2$ as in \cite{sobolev-gan}.
We choose $\lambda$ and an overall scaling to obtain the losses:
\[
  \text{SMMD: }
  \frac{\widehat{\MMD}_{k_\psi}^2(\PP, \QQ_\theta)}{1 + 10 \E_{\hat\PP}\left[ \lVert\nabla \phi_\psi(X) \rVert_F^2 \right]}
  ,
  \text{ SWGAN: }
  \frac{\E_{\hat\PP}\left[ \phi_\psi(X) \right] - \E_{\hat\QQ_\theta}\left[ \phi_\psi(X) \right]}{\sqrt{1 + 10 \left( \E_{\hat\PP}\left[ \lvert \phi_\psi(X) \rvert^2 \right] + \E_{\hat\PP}\left[ \lVert \nabla \phi_\psi(X) \rVert_F^2 \right] \right)}}
.\]

\textbf{Architecture} %
For CIFAR-10, we used the CNN architecture proposed by \cite{Miyato:2018}
with a 7-layer critic and a 4-layer generator.
For CelebA, we used a 5-layer DCGAN discriminator and a 10-layer ResNet generator as in \parencite{Binkowski:2018}.
For ImageNet, we used a 10-layer ResNet for both the generator and discriminator.
In all experiments we used $64$ filters for the smallest convolutional layer,
and double it at each layer (CelebA/ImageNet) or every other layer (CIFAR-10).
The input codes for the generator are drawn from $\mathrm{Uniform}\left( [-1, 1]^{128} \right)$.
We consider two parameterizations for each critic:
a standard one where the parameters can take any real value,
and a spectral parametrization (denoted SN-)
as above \parencite{Miyato:2018}.
Models without explicit gradient control
(SN-GAN, SN-MMDGAN, SN-MMGAN-L2, SN-WGAN)
fix $\gamma = 1$, for spectral normalization;
others learn $\gamma$, using a spectral parameterization.

\textbf{Training}
All models were trained for $150\,000$ generator updates on a single GPU,
except for ImageNet where the model was trained on 3 GPUs simultaneously.
To limit communication overhead %
we averaged the MMD estimate on each GPU,
giving the block MMD estimator \parencite{b-test}.
We always used $64$ samples per GPU from each of $\PP$ and $\QQ$,
and $5$ critic updates per generator step.
We used initial learning rates of $0.0001$ for CIFAR-10 and CelebA,
$0.0002$ for ImageNet,
and decayed these rates using the KID adaptive scheme of \cite{Binkowski:2018}:
every $2\,000$ steps, generator samples are compared to those from $20\,000$ steps ago,
and if the relative KID test \parencite{3sample} fails to show an improvement three consecutive times,
the learning rate is decayed by $0.8$.
We used the Adam optimizer \parencite{adam} with $\beta_1 = 0.5$, $\beta_2 = 0.9$.

\textbf{Evaluation}
To compare the sample quality of different models,
we considered three different scores
based on the Inception network \parencite{inception} trained for ImageNet classification,
all using default parameters in the implementation of \cite{Binkowski:2018}.
The \emph{Inception Score (IS)} \parencite{improved-gans}
is based on the entropy of predicted labels;
higher values are better.
Though standard, this metric has many issues,
particularly on datasets other than ImageNet \parencite{note-on-inception,fid,Binkowski:2018}.
The \emph{FID} \parencite{fid}
instead measures the similarity of samples from the generator and the target
as the Wasserstein-2 distance between Gaussians fit to
their intermediate representations.
It is more sensible than the IS and becoming standard,
but its estimator is strongly biased \parencite{Binkowski:2018}.
The \emph{KID} \parencite{Binkowski:2018}
is similar to FID,
but by using a polynomial-kernel MMD its estimates enjoy better statistical properties
and are easier to compare.
(A similar score was recommended by \cite{empirical-evaluation}.)

 \begin{figure}[p]
  \centering
    \includegraphics[width=1.\linewidth]{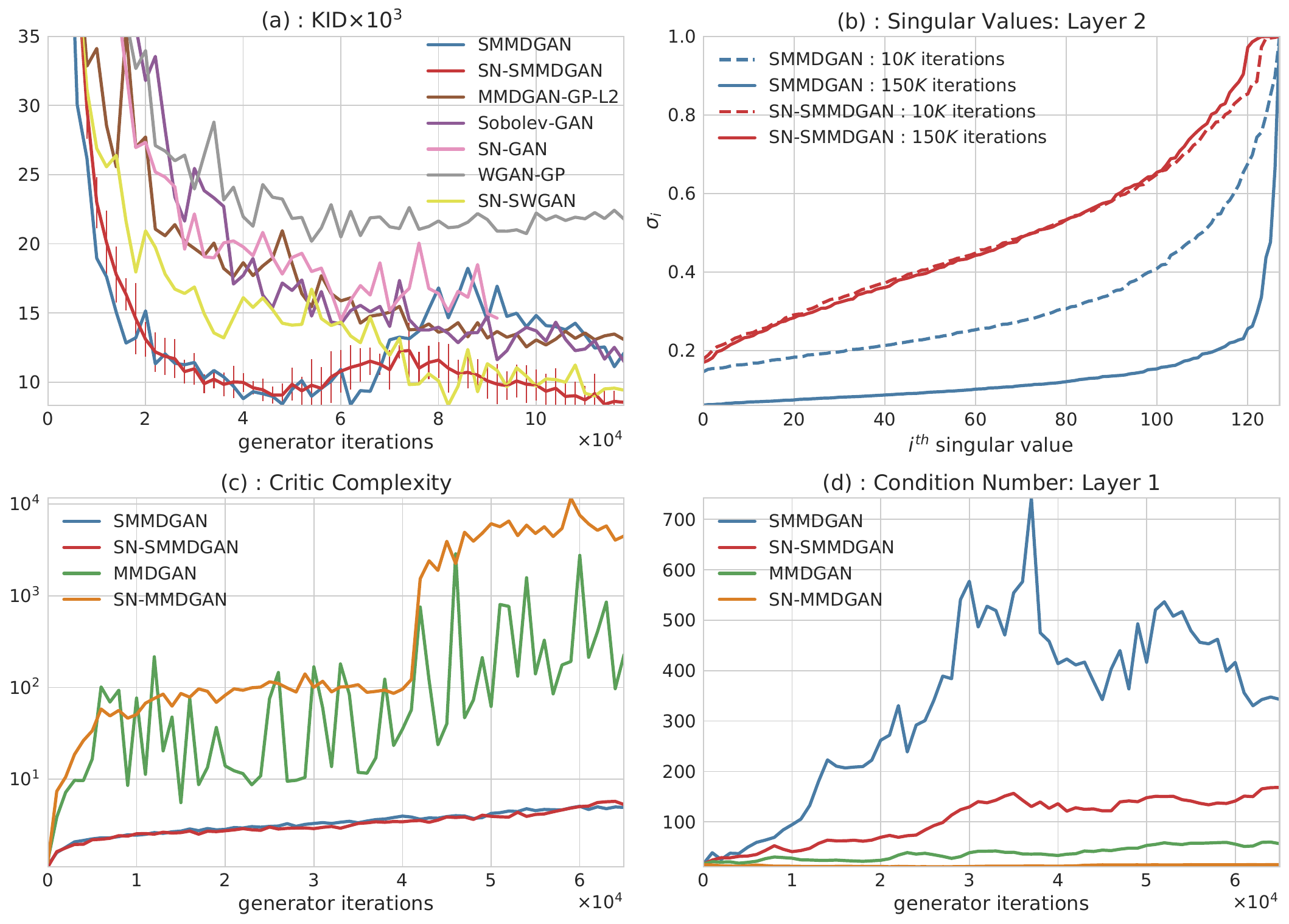}
    \caption{The training process on CelebA.
      (a) KID scores. We report a final score for SN-GAN slightly before its sudden failure mode;
          MMDGAN and SN-MMDGAN were unstable and had scores around 100.
      (b) Singular values of the second layer, both early (dashed) and late (solid) in training.
      (c) $\sigma_{\mu,k,\lambda}^{-2}$ for several MMD-based methods.
      (d) The condition number in the first layer through training.
      SN alone does not control $\sigma_{\mu,k,\lambda}$,
      and SMMD alone does not control the condition number.
    }
    \label{fig:celebA_scores_and_singular_values}
 \end{figure}

\begin{table}[ht]
  \centering
  \caption{Mean (standard deviation) of score estimates, based on $50\,000$ samples from each model.} %
  \label{tab:scores}
  \begin{subtable}[t]{\linewidth}
    \centering
    \caption{CIFAR-10 and CelebA.}
    \label{tab:celebA_cifar10_scores}
    \input{tables/scores_celebA_cifar10}
  \end{subtable}
  \begin{subtable}[t]{\linewidth}
    \centering
    \caption{ImageNet.}
    \label{tab:imagenet_scores}
    \input{tables/scores_imagenet}
  \end{subtable}
\end{table}

\begin{figure}[p]
    \centering
    \begin{subfigure}[t]{0.30\textwidth}
        \centering
        \includegraphics[width=\linewidth]{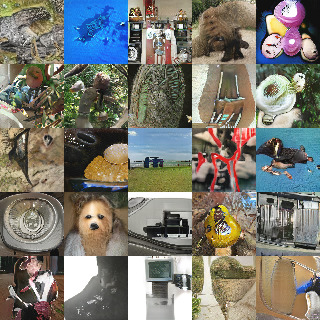}
        \caption{Scaled MMD GAN with SN}
        \label{fig:imagenet_sn_smmd}
    \end{subfigure}
    \hfill
    \begin{subfigure}[t]{0.30\textwidth}
        \centering
        \includegraphics[width=\linewidth]{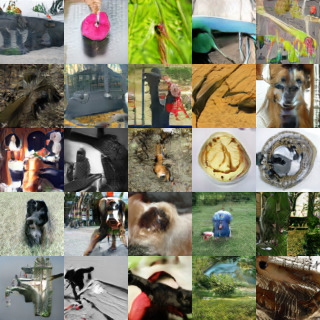}
        \caption{SN-GAN} \label{fig:imagenet_sngan}
    \end{subfigure}
    \hfill
    \begin{subfigure}[t]{0.30\textwidth}
        \centering
        \includegraphics[width=\linewidth]{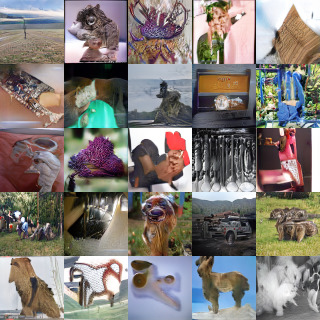}
        \caption{Boundary Seeking GAN} \label{fig:imagenet_bgan}
    \end{subfigure}

    \begin{subfigure}[t]{0.30\textwidth}
        \centering
        \includegraphics[width=\linewidth]{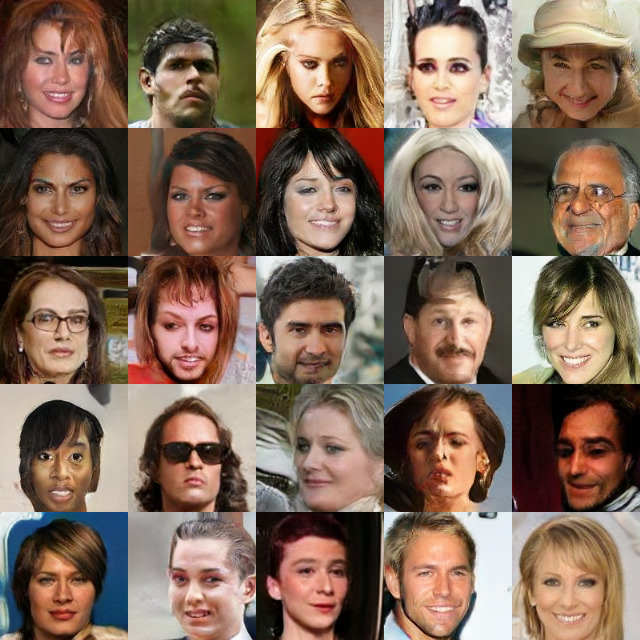}
        \caption{Scaled MMD GAN with SN} \label{fig:celebA_sn_smmd}
    \end{subfigure}
    \hfill
    \begin{subfigure}[t]{0.30\textwidth}
        \centering
        \includegraphics[width=\linewidth]{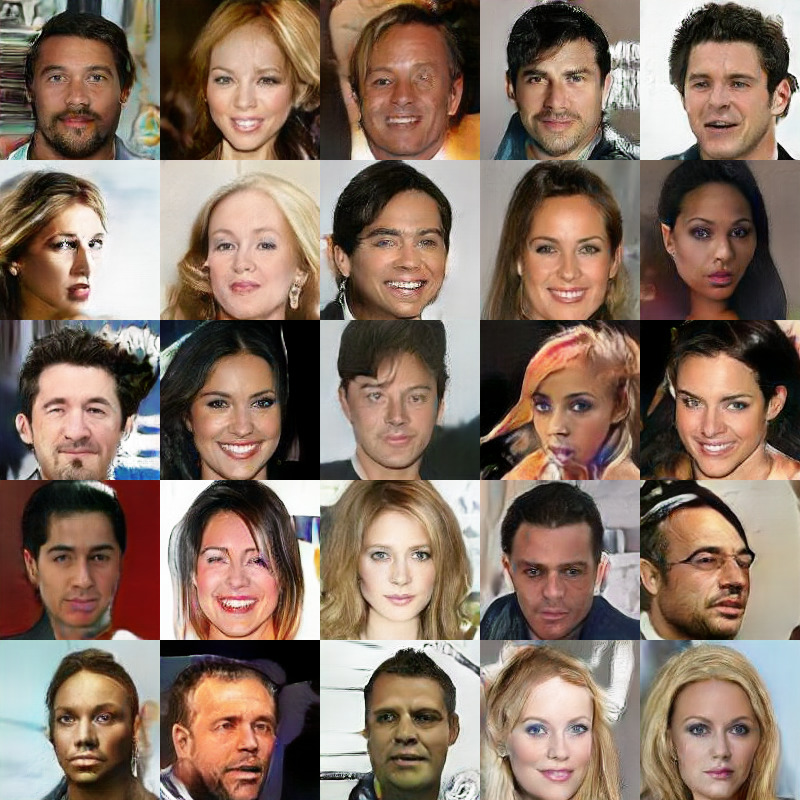}
        \caption{Scaled WGAN with SN} \label{fig:celebA_sn_swgan}
    \end{subfigure}
    \hfill
    \begin{subfigure}[t]{0.30\textwidth}
        \centering
        \includegraphics[width=\linewidth]{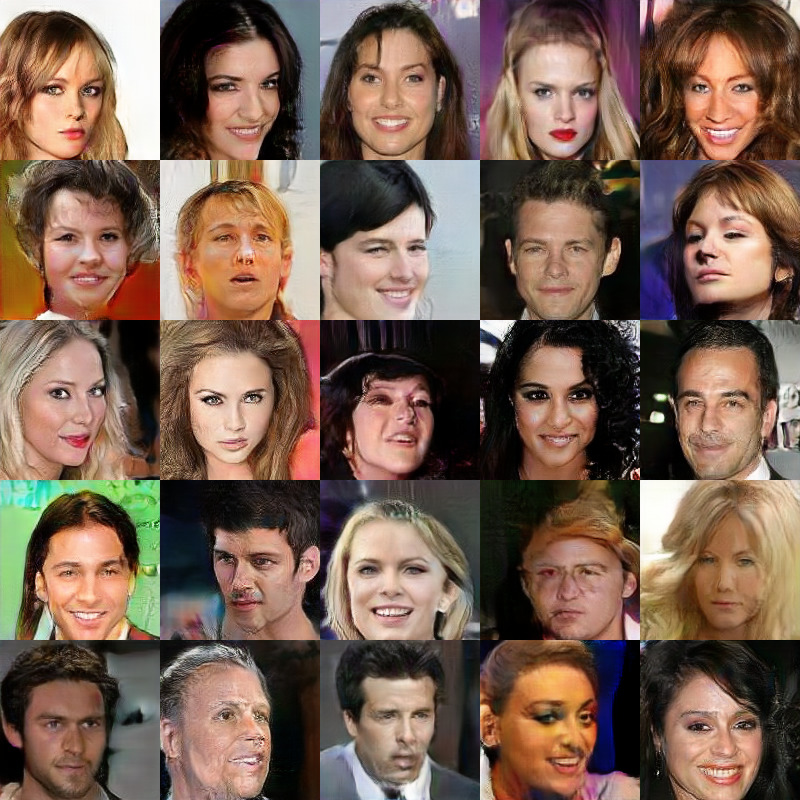}
        \caption{MMD GAN with GP+L2} \label{fig:celebA_mmd_gp}
    \end{subfigure}
    \caption{Samples from various models. Top: $64 \times 64$ ImageNet; bottom: $160 \times 160$ CelebA.}
    \label{fig:samples}
\end{figure}

\textbf{Results}
\cref{tab:celebA_cifar10_scores} presents the scores for models trained on both CIFAR-10 and CelebA datasets. On CIFAR-10,  SN-SWGAN and  SN-SMMDGAN performed comparably to SN-GAN. But on CelebA,
SN-SWGAN and SN-SMMDGAN dramatically outperformed the other methods with the same architecture in all three metrics. It also trained faster, and consistently outperformed other methods over multiple initializations (\cref{fig:celebA_scores_and_singular_values} (a)).
It is worth noting that SN-SWGAN far outperformed WGAN-GP on both datasets.
\cref{tab:imagenet_scores} presents the scores for SMMDGAN and SN-SMMDGAN trained on ImageNet,
and the scores of pre-trained models using BGAN \citep{began} and SN-GAN \citep{Miyato:2018}.\footnote{These models are courtesy of the respective authors and also trained at $64 \times 64$ resolution. SN-GAN used the same architecture as our model, but trained for $250\,000$ generator iterations; BS-GAN used a similar 5-layer ResNet architecture and trained for 74 epochs, comparable to SN-GAN.}
The proposed methods substantially outperformed both methods in FID and KID scores.
\cref{fig:samples} shows samples on ImageNet and CelebA;
\cref{appendix:additional-samples} has more.

\textbf{Spectrally normalized WGANs / MMDGANs}
To control for the contribution of the spectral parametrization to the performance,
we evaluated variants of MMDGANs, WGANs and Sobolev-GAN using spectral normalization
(in \cref{tab:sn_cifar10_scores}, \cref{appendix:smmd_vs_sn}).
WGAN and Sobolev-GAN led to unstable training and didn't converge at all (\cref{fig:score_per_iter_cifar10_sn}) despite many attempts.
MMDGAN converged on CIFAR-10 (\cref{fig:score_per_iter_cifar10_sn})  but was unstable on CelebA (\cref{fig:loss_celebA}).
The gradient control due to SN is thus probably too loose for these methods.
This is reinforced by \cref{fig:celebA_scores_and_singular_values} (c),
which shows that the expected gradient of the critic network
is much better-controlled by SMMD, even when SN is used.
We also considered variants of these models with a learned $\gamma$
while also adding a gradient penalty and an $L_2$ penalty on critic activations \citep[footnote 19]{Binkowski:2018}.
These generally behaved similarly to MMDGAN, and didn't lead to substantial improvements.
We ran the same experiments on CelebA,
but aborted the runs early when it became clear that training was not successful.

\textbf{Rank collapse}
We occasionally observed the failure mode for SMMD where the critic becomes low-rank,
discussed in \cref{par:parametrization},
especially on CelebA;
this failure was obvious even in the training objective.
\Cref{fig:celebA_scores_and_singular_values} (b) is one of these examples.
Spectral parametrization seemed to prevent this behavior.
We also found one could avoid collapse by reverting to an earlier checkpoint
and increasing the RKHS regularization parameter $\lambda$,
but did not do this for any of the experiments here.

\section{Conclusion}
We studied gradient regularization for MMD-based critics in implicit generative models,
clarifying how previous techniques relate to the $\optMMD$ loss.
Based on these insights,
we proposed the Gradient-Constrained MMD and its approximation the Scaled MMD,
a new loss function for IGMs that controls gradient behavior in a principled way
and obtains excellent performance in practice.

One interesting area of future study for these distances is their behavior
when used to diffuse particles distributed as $\QQ$ towards particles distributed as $\PP$.
\Textcite[Appendix A.1]{sobolev-gan} began such a study for the Sobolev GAN loss; %
\cite{sobolev-descent} proved convergence and studied discrete-time approximations.

Another area to explore is the geometry of these losses,
as studied by \textcite{Bottou:2017},
who showed potential advantages of the Wasserstein geometry over the MMD.
Their results, though, do not address any distances based on optimized kernels;
the new distances introduced here
might have interesting geometry of their own.

\printbibliography

\clearpage
\appendix

\section{Proofs} \label{appendix:proofs}

We first review some basic properties of Reproducing Kernel Hilbert
Spaces. We consider here a separable RKHS $\h$ with basis $(e_{i})_{i\in I}$,
where $I$ is either finite if $\h$ is finite-dimensional, or $I=\mathbb{N}$
otherwise. We also assume that the reproducing kernel $k$ is continuously
twice differentiable.

We use a slightly nonstandard notation for derivatives:
$\partial_i f(x)$ denotes the $i$th partial derivative of $f$ evaluated at $x$,
and $\partial_i \partial_{j+d} k(x, y)$
denotes $\frac{\partial^2 k(a, b)}{\partial a_i \partial b_j} \vert_{(a,b) = (x, y)}$.

Then the following reproducing properties hold for
any given function $f$ in $\h$ \citep[Lemma 4.34]{SteChr08}:
\begin{align}
f(x)= & \langle f,k(x,.)\rangle_{\h}\label{eq:reproducing_prop}\\
\partial_{i}f(x)= & \langle f,\partial_{i}k(x,.)\rangle_{\h}\label{eq:reproducing_derivative}
.\end{align}

We say that an operator $A:\h\mapsto\h$ is Hilbert-Schmidt if $\Vert A\Vert_{HS}^{2}=\sum_{i\in I}\Vert Ae_{i}\Vert_{\h}^{2}$
is finite. $\Vert A\Vert_\hs$ is called the Hilbert-Schmidt norm
of $A$. The space of Hilbert-Schmidt operators itself a Hilbert space
with the inner product $\langle A,B\rangle_{HS}=\sum_{i\in I}\langle Ae_{i},Be_{i}\rangle_{\h}$.
Moreover, we say that an operator $A$ is trace-class if its trace
norm is finite, i.e. $\Vert A\Vert_{1}=\sum_{i\in I}\langle e_{i},(A^{*}A)^{\frac{1}{2}}e_{i}\rangle_{\h}<\infty$.
The outer product $f \otimes g$ for $f, g \in \h$
gives an $\h \to \h$ operator such that
$(f \otimes g) v = \langle g, v\rangle_{\h} f$ for all $v$ in $\h$.

Given two vectors $f$ and $g$ in $\h$ and a Hilbert-Schmidt operator
$A$ we have the following properties:
\begin{proplist}
  \item \label{HS_norm_rank_one}
    The outer product $f\otimes g$ is a Hilbert-Schmidt operator with Hilbert-Schmidt norm given by: $\Vert f\otimes g\Vert_\hs = \Vert f\Vert_{\h}\Vert g\Vert_{\h}$.
  \item \label{HS_inner_rank_one}
    The inner product between two rank-one operators $f\otimes g$ and $u\otimes v$ is
    $\langle f\otimes g,u\otimes v\rangle_\hs=\langle f,u\rangle_{\h}\langle g,v\rangle_{\h}$.
  \item \label{HS_inner_prod_rank_one}
    The following identity holds: $\langle f, Ag\rangle_\h = \langle f\otimes g, A\rangle_{\hs}$.
\end{proplist}

Define the following covariance-type operators:
\begin{equation}
     D_x
   = k(x, \cdot) \otimes k(x, \cdot) + \sum_{i=1}^d \partial_i k(x, \cdot) \otimes \partial_i k(x, \cdot)
  \quad
     D_\mu
   = \E_{X \sim \mu} D_X
  \quad
     D_{\mu,\lambda}
   = D_\mu + \lambda I
\label{eq:d-op}
;\end{equation}
these are useful in that, using \eqref{eq:reproducing_prop} and \eqref{eq:reproducing_derivative},
$\langle f, D_x g \rangle = f(x) g(x) + \sum_{i=1}^d \partial_i f(x) \, \partial_i g(x)$.

\subsection{Definitions and estimators of the new distances} \label{sec:proofs:distances}

We will need the following assumptions about the distributions $\PP$ and $\QQ$,
the measure $\mu$,
and the kernel $k$:
\begin{assumplist}
  \item \label{Moments} $\PP$ and $\QQ$ have integrable first moments.
  \item \label{Growth} $\sqrt{k(x, x)}$ grows at most linearly in $x$: for all $x$ in $\x$, $\sqrt{k(x, x)} \leq C (\lVert x \rVert + 1)$ for some constant $C$.
  \item \label{Differentiability} The kernel $k$ is twice continuously differentiable.
  \item \label{Integrability} The functions $x \mapsto k(x, x)$ and $x \mapsto \partial_i \partial_{i+d} k(x, x)$ for $1\leq i\leq d$ are $\mu$-integrable.
\end{assumplist}
When $k = K \circ \phi_\psi$, \cref{Growth} is automatically satisfied by a $K$ such as the Gaussian;
when $K$ is linear, it is true for a quite general class of networks $\phi_\psi$ \citep[Lemma 1]{Binkowski:2018}.

We will first give a form for the Gradient-Constrained MMD \eqref{eq:SobolevMMD}
in terms of the operator \eqref{eq:d-op}:
\begin{prop} \label{prop:dot_prod_expression}
Under \cref{Moments,Growth,Differentiability,Integrability},
the Gradient-Constrained MMD is given by
 \begin{equation}
   \GCMMD_{\mu,k,\lambda}(\PP, \QQ)
   = \sqrt{\langle\eta, D_{\mu,\lambda}^{-1} \eta \rangle_{\h}}
   \label{eq:GCMMD_operator_form}
 .\end{equation}
\end{prop}
\begin{proof}[Proof of \cref{prop:dot_prod_expression}]
Let $f$ be a function in $\h$.
We will first express the squared $\lambda$-regularized Sobolev norm of $f$
\eqref{eq:sobolev-norm}
as a quadratic form in $\h$.
Recalling the reproducing properties
of \eqref{eq:reproducing_prop} and \eqref{eq:reproducing_derivative},
we have:
\[
  \Vert f\Vert_{S(\mu),k,\lambda}^{2}
  = \int\langle f,k(x, \cdot)\rangle_{\h}^{2} \, \mudx
  + \sum_{i=1}^{d} \int\langle f, \partial_{i} k(x,\cdot) \rangle_{\h}^{2} \, \mudx
  + \lambda \lVert f \rVert_{\h}^{2}
.\]
Using \cref{HS_inner_rank_one} and the operator \eqref{eq:d-op}, one further gets
\[
  \Vert f\Vert_{S(\mu),k,\lambda}^{2}
  = \int\langle f \otimes f, D_x \rangle_\hs \, \mudx + \lambda \Vert f\Vert_{\h}^{2}
.\]
Under \cref{Integrability}, and using \cref{lem:Bochner_interversion},
one can take the integral inside the inner product, which leads to
$\Vert f\Vert_{S(\mu),k,\lambda}^{2} = \langle f\otimes f,D_\mu\rangle_\hs + \lambda\Vert f\Vert_{\h}^{2}$.
Finally, using \cref{HS_inner_prod_rank_one} it follows that
\[ \Vert f\Vert_{S(\mu),k,\lambda}^{2}=\langle f,D_{\mu,\lambda}f\rangle_{\h} .\]

Under \cref{Moments,Growth}, \cref{lem:Bochner_interversion} applies,
and it follows that $k(x, \cdot)$ is also Bochner integrable under $\PP$ and $\QQ$.
Thus
\[
    \E_\PP\left[ \langle f, k(x, \cdot)\rangle_{\h}\right]
  - \E_\QQ\left[ \langle f, k(x, \cdot)\rangle_{\h}\right]
  = \langle f, \E_\PP\left[ k(x, \cdot) \right] - \E_\PP\left[k(x, \cdot)\right] \rangle_{\h}
  = \langle f, \eta \rangle_{\h}
,\]
where $\eta$ is defined as this difference in mean embeddings.

Since $D_{\mu,\lambda}$ is symmetric positive definite,
its square-root $D_{\mu,\lambda}^{\frac{1}{2}}$ is well-defined and is also invertible.
For any $f \in \h$,
let $g = D_{\mu,\lambda}^{\frac{1}{2}} f$,
so that $\langle f,D_{\mu,\lambda}f\rangle_{\h}=\Vert g\Vert_{\h}^{2}$.
Note that for any $g \in \h$, there is a corresponding $f = D_{\mu,\lambda}^{-\frac12} g$.
Thus we can re-express the maximization problem in \eqref{eq:SobolevMMD} in terms of $g$:
\begin{align*}
\GCMMD_{\mu,k,\lambda}(\PP,\QQ)
 &:= \sup_{\substack{f \in \h \\ \langle f, D_{\mu,\lambda} f \rangle_\h \leq 1}}
     \langle f, \eta \rangle_{\h}
   =\sup_{\substack{g \in \h \\ \lVert g \rVert_\h \le 1}}
     \langle D_{\mu,\lambda}^{-\frac{1}{2}} g, \eta \rangle_{\h}
\\&= \sup_{\substack{g \in \h \\ \Vert g \Vert_\h \le 1}}
     \langle g, D_{\mu,\lambda}^{-\frac{1}{2}} \eta \rangle_{\h}
   = \lVert D_{\mu,\lambda}^{-\frac{1}{2}} \eta \rVert_{\h}
   = \sqrt{\langle\eta, D_{\mu,\lambda}^{-1} \eta \rangle_{\h}}
.\qedhere\end{align*}
\end{proof}

\Cref{prop:dot_prod_expression}, though, involves inverting the infinite-dimensional operator $D_{\mu,\lambda}$
and thus doesn't directly give us a computable estimator.
\Cref{prop:Finite_rank_approx} solves this problem in the case where $\mu$ is a discrete measure:
\begin{repprop}{prop:Finite_rank_approx}
  Let $\hat\mu = \sum_{m=1}^M \delta_{X_m}$ be an empirical measure of $M$ points.
  Let $\eta(X) \in \R^M$ have $m$th entry $\eta(X_m)$,
  and $\nabla \eta(X) \in \R^{M d}$ have $(m, i)$th entry\footnote{%
    We use $(m, i)$ to denote $(m-1) d + i$;
    thus $\nabla \eta(X)$ stacks $\nabla \eta(X_1)$, \dots, $\nabla \eta(X_M)$ into one vector.
  } $\partial_i \eta(X_m)$.
  Then under \cref{Moments,Growth,Differentiability,Integrability},
  the Gradient-Constrained MMD is
  \begin{align*}
    \GCMMD_{\hat\mu,k,\lambda}^2(\PP, \QQ)
    &= \frac1\lambda \left( \MMD^{2}(\PP, \QQ) - \bar{P}(\eta) \right)
  \\
    \bar{P}(\eta)
    &=
    \begin{bmatrix} \eta(X) \\ \nabla\eta(X) \end{bmatrix}\tp
    \left(
      \begin{bmatrix}
        K & G\tp \\
        G & H
      \end{bmatrix}
      + M \lambda I_{M + M d}
    \right)^{-1}
    \begin{bmatrix} \eta(X) \\ \nabla\eta(X) \end{bmatrix}
    \label{eq:Sobolev-penalty}
  ,\end{align*}
  where $K$ is the kernel matrix $K_{m,m'} = k(X_m, X_{m'})$,
  $G$ is the matrix of left derivatives %
  $G_{(m, i), m'} = \partial_i k(X_m, X_{m'})$,
  and $H$ that of derivatives of both arguments $H_{(m, i), (m', j)} = \partial_i \partial_{j+d} k(X_m, X_{m'})$.
\end{repprop}

Before proving \cref{prop:Finite_rank_approx},
we note the following interesting alternate form.
Let $\bar e_i$ be the $i$th standard basis vector for $\R^{M + M d}$,
and define $T : \h \to \R^{M + M d}$ as the linear operator
\[
  T = \sum_{m=1}^M \bar e_m \otimes k(X_m, \cdot)
    + \sum_{m=1}^M \sum_{i=1}^d \bar e_{m + (m, i)} \otimes \partial_i k(X_m, \cdot)
.\]
Then $\begin{bmatrix}\eta(X) \\ \nabla \eta(X) \end{bmatrix} = T \eta$,
and $\begin{bmatrix}K & G\tp \\ G & H \end{bmatrix} = T T^*$.
Thus we can write
\[
  \GCMMD_{\hat\mu,k,\lambda}^2
  = \frac1\lambda \left\langle \eta, \left( I - T^* (T T^* + M \lambda I)^{-1} T \right) \eta \right\rangle_\h
.\]

\begin{proof}[Proof of \cref{prop:Finite_rank_approx}]
\label{proof:Finite_rank_approx}

Let $g \in \h$ be the solution to the regression problem $D_{\mu,\lambda} g = \eta$:
\begin{gather}
  \frac1M \sum_{m=1}^M\left[
    g(X_m) k(X_m, \cdot)
    + \sum_{i=1}^d \partial_i g(X_m) \partial_i k(X_m, \cdot)
  \right]
  + \lambda g
  = \eta
  \nonumber
  \\
  g =
  \frac1\lambda \eta
  - \frac{1}{\lambda M} \sum_{m=1}^M\left[
    g(X_m) k(X_m, \cdot)
    + \sum_{i=1}^d \partial_i g(X_m) \partial_i k(X_m, \cdot)
  \right]
  \label{eq:Functional_Regression}
.\end{gather}
Taking the inner product of both sides of \eqref{eq:Functional_Regression}
with $k(X_{m'}, \cdot)$ for each $1 \le m' \le M$
yields the following $M$ equations:
\begin{equation}
  g(X_{m'})
  = \frac{1}{\lambda}\eta(X_{m'})
  - \frac{1}{\lambda M} \sum_{m=1}^{M} \left[
      g(X_{m}) K_{m,m'} %
      + \sum_{i=1}^{d} \partial_{i} g(X_{m}) \, G_{(m, i), m'} %
    \right]
  \label{eq:system_g}
.\end{equation}
Doing the same with $\partial_j k(X_{m'}, \cdot)$ gives $M d$ equations:
\begin{equation}
  \partial_{j}g(X_{m'})
  = \frac{1}{\lambda} \partial_{j} \eta(X_{m'})
  - \frac{1}{\lambda M} \sum_{m=1}^{M} \left[
      g(X_{m}) G_{(m', j), m}  %
      + \sum_{i=1}^{d} \partial_{i} g(X_{m}) H_{(m,i), (m', j)} %
  \right]
  \label{eq:system_grad_g}
.\end{equation}
From \eqref{eq:Functional_Regression}, it is clear that $g$ is a linear
combination of the form:
\[
  g(x)
  = \frac{1}{\lambda} \eta(x)
  - \frac{1}{\lambda M} \sum_{m=1}^{M}\left[
      \alpha_{m} k(X_{m}, x)
      + \sum_{i=1}^{d} \beta_{m,i} \partial_{i} k(X_{m}, x)
  \right]
,\]
where the coefficients $\alpha:=\left(\alpha_{m}=g(X_{m})\right)_{1\leq m\leq M}$
and $\beta:=\left(\beta_{m,i}=\partial_{i}g(X_{m})\right)_{\substack{1\leq m\leq M\\
1\leq i\leq d
}
}$ satisfy the system of equations \eqref{eq:system_g} and \eqref{eq:system_grad_g}.
We can rewrite this system as
\[
  \begin{bmatrix}
    K+M\lambda I_{M} & G\tp \\
    G & H+M\lambda I_{Md}
  \end{bmatrix}
  \begin{bmatrix}\alpha\\\beta\end{bmatrix}
  = M \begin{bmatrix} \eta(X) \\ \nabla\eta(X) \end{bmatrix}
,\]
where $I_{M}$, $I_{Md}$ are the identity matrices of dimension $M$, $Md$.
Since $K$ and $H$ must be positive semidefinite, an inverse exists.
We conclude by noticing that
\[
  \GCMMD_{\hat\mu,k,\lambda}(\PP, \QQ)^2
  = \langle \eta, g \rangle_{\h}
  = \frac{1}{\lambda} \lVert\eta\rVert_{\h}^{2}
  - \frac{1}{\lambda M} \sum_{m=1}^{M}\left[
    \alpha_{m} \eta(X_{m}) + \sum_{i=1}^{d} \beta_{m,i} \partial_{i}\eta(X_{m})
  \right]
.\qedhere \]
\end{proof}

The following result was key to our definition of the SMMD in \cref{subsec:Scaled-MMD}.
\begin{repprop}{prop:Sobolev_Upperbound}
  Under \cref{Moments,Growth,Differentiability,Integrability},
  we have for all $f \in \h$ that
  \[
    \lVert f \rVert_{S(\mu),k,\lambda}
    \le \sigma_{\mu,k,\lambda}^{-1} \lVert f \rVert_{\h_k}
  ,\]
  where
  $\sigma_{k,\mu,\lambda} := 1 / \sqrt{
    \lambda + \int k(x, x) \mu(\ud x) + \sum_{i=1}^d \int \partial_i \partial_{i+d} k(x, x) \mu(\ud x)
  }$.
\end{repprop}
\begin{proof}[Proof of \cref{prop:Sobolev_Upperbound}]
\label{proof:Sobolev_Upperbound}
The key idea here is to use the Cauchy-Schwarz inequality for the Hilbert-Schmidt inner product.
Letting $f \in \h$,
$\Vert f\Vert_{S(\mu),k,\lambda}^{2}$ is
\begin{align*}
  &  \int f(x)^{2}\mudx
   + \int\Vert\nabla f(x)\Vert^{2}\mudx
   + \lambda \lVert f \rVert_{\h}^{2}
  \\
  &\stackrel{(a)}{=}
    \int\langle f, k(x, \cdot)\otimes k(x, \cdot) f \rangle_{\h} \mudx
  + \sum_{i=1}^{d} \int \langle f, \partial_{i} k(x, \cdot)\otimes\partial_{i}k(x, \cdot) f \rangle_{\h} \mudx
  + \lambda \lVert f \rVert_{\h}^{2}
  \nonumber \\
  &\stackrel{(b)}{=}
    \int\langle f\otimes f,k(x, \cdot)\otimes k(x, \cdot)\rangle_\hs \,\mudx
  + \sum_{i=1}^d \int \langle f\otimes f, \partial_i k(x, \cdot) \otimes \partial_i k(x, \cdot)\rangle_\hs \mudx
  + \lambda \lVert f \rVert_{\h}^{2}
  \nonumber \\
  &\stackrel{(c)}{\leq}
    \lVert f \rVert_{\h}^{2}\left[
      \int k(x,x) \mudx
      + \sum_{i=1}^d \int \partial_i \partial_{i+d} k(x, x) \mudx
      + \lambda
    \right]
  \nonumber
.\end{align*}
$(a)$ follows from the reproducing properties \eqref{eq:reproducing_prop}
and \eqref{eq:reproducing_derivative} and \cref{HS_inner_rank_one}.
$(b)$ is obtained using \cref{HS_inner_prod_rank_one},
while $(c)$ follows from the Cauchy-Schwarz inequality and \cref{HS_norm_rank_one}.
\end{proof}
\begin{lem}
\label{lem:Bochner_interversion}Under \cref{Integrability}, $D_{x}$
is Bochner integrable and its integral $D_\mu$ is a trace-class symmetric
positive semi-definite operator with $D_{\mu,\lambda}=D+\lambda I$ invertible
for any positive $\lambda$. Moreover, for any Hilbert-Schmidt operator
$A$ we have: $\langle A,D_\mu \rangle_{HS}=\int\langle A,D_{x}\rangle_{HS}\mudx$.

Under \cref{Moments,Growth}, $k(x, \cdot)$ is Bochner integrable with
respect to any probability distribution $\PP$ with finite first moment
and the following relation holds:
$\langle f,\E_{\PP}\left[k(x, \cdot)\right]\rangle_{\h} = \E_{\PP}\left[\langle f,k(x, \cdot)\rangle_{\h}\right]$
for all $f$ in $\h$.
\end{lem}
\begin{proof}
The operator $D_{x}$ is positive self-adjoint.
It is also trace-class,
as by the triangle inequality
\begin{align*}
  \Vert D_{x} \Vert_{1}
  &\le \Vert k(x, \cdot) \otimes k(x, \cdot)\Vert_{1}
     + \sum_{i=1}^{d} \Vert \partial_i k(x, \cdot) \otimes \partial_i k(x, \cdot) \Vert_{1}
  \\
  &= \lVert k(x, \cdot) \Vert_{\h}^{2}
   + \sum_{i=1}^{d} \lVert \partial_i k(x, \cdot) \rVert_{\h}^{2}
   <\infty
.\end{align*}
By \cref{Integrability},
we have that $\int \Vert D_{x}\Vert_1\mudx<\infty$
which implies that $D_{x}$ is $\mu$-integrable in the Bochner sense
\parencite[Definition 1 and Theorem 2]{Retherford:1978}. %
Its integral $D_\mu$ is trace-class and satisfies
$\Vert D_\mu\Vert_{1}\le\int\Vert D_{x}\Vert_{1} \mudx$.  This allows to have $\langle A,D_\mu \rangle_{HS} = \int \langle A, D_x \rangle_{HS} \mudx$ for all Hilbert-Schmidt operators $A$. Moreover, the integral preserves the symmetry and positivity. It follows that $\D_{\mu,\lambda}$ is invertible.

The Bochner integrability of $k(x,\cdot)$ under a distribution $\PP$ with finite moment follows directly from \cref{Moments,Growth}, since $ \int \Vert k(x, \cdot)\Vert \operatorname{\PP}(\dx)  \leq C\int (\Vert x\Vert+1) \operatorname{\PP}(\dx) < \infty$. This allows us to write $ \langle f, \E_{\PP}[k(x, \cdot)]\rangle_{\h} = \E_{\PP}[\langle f, k(x, \cdot) \rangle_{\h}] $.
\end{proof}

\subsection{Continuity of the Optimized Scaled MMD in the Wasserstein topology}\label{appendix:continuity_opt_smmd}
To prove \cref{thm:continuity_opt_SMMD},
we we will first need some new notation.

We assume the kernel is $k = K \circ \phi_\psi$,
i.e.\ $k_\psi(x, y) = K(\phi_\psi(x), \phi_\psi(y))$,
where the representation function $\phi_\psi$ is a network $\phi_\psi(X) : \R^d \to \R^{d_L}$
consisting of $L$ fully-connected layers:
\begin{align}
  h_{\psi}^{0}(X) &= X\nonumber
  \\
  h_{\psi}^{l}(X) &= W^l \sigma_{l-1}(h_{\psi}^{l-1}(X)) + b^{l}
  \qquad \text{ for } 1 \le l \leq L
  \label{eq:network}
  \\
  \phi_\psi(X) &= h_\psi^L(X)
  \nonumber
.\end{align}
The intermediate representations $h_\psi^l(X)$ are of dimension $d_l$,
the weights $W^{l}$ are matrices in $\R^{d_{l} \times d_{l-1}}$,
and biases $b^{l}$ are vectors in $\R^{d_{l}}$.
The elementwise activation function $\sigma$ is given by
$\sigma_0(x) = x$,
and for $l > 0$ the activation $\sigma_{l}$ is a leaky ReLU with
leak coefficient $0<\alpha<1$:
\begin{equation}
\sigma_{l}(x)= \sigma(x)=\begin{cases}
x & x>0\\
\alpha x & x\leq0
\end{cases}
\qquad \text{ for } l>0
\label{eq:lReLU}
.\end{equation}

The parameter $\psi$ is the concatenation of all the layer parameters:
\[
  \psi = \left( (W^{L},b^{L}), (W^{L-1},b^{L-1}), \dots, (W^{1}, b^{1}) \right)
.\]
We denote by $\Psi$ the set of all such possible parameters,
i.e.\ $\Psi = \R^{d_L \times d_{L-1}} \times \R^{d_L} \times \cdots \times \R^{d_1 \times d} \times \R^{d_1}$.
Define the following restrictions of $\Psi$:
\begin{align}
  \Psi^{\kappa} :=& \left\{
    \psi \in \Psi \mid \forall 1 \le l \le L, \;
    \cond(W^l)\leq \kappa
  \right\}
  \label{eq:def:psi-kappa}
  \\
  \Psi^{\kappa}_1 :=& \left\{
    \psi \in \Psi^{\kappa} \mid \forall 1 \le l \le L, \;
    \lVert W^l \rVert = 1
  \right\}
  \label{eq:def:psi-kappa-1}
.\end{align}
$\Psi^{\kappa}$ is the set of those parameters such that $W^l$ have a small condition number,
$\cond(W) = \sigma_{\max}(W) / \sigma_{\min}(W)$.
$\Psi_1^{\kappa}$ is the set of per-layer normalized parameters with a condition number bounded by $\kappa$.

Recall the definition of Scaled MMD, \cref{eq:Scaled_MMD},
where $\lambda > 0$ and $\mu$ is a probability measure:
\begin{align*}
  \SMMD_{\mu,k,\lambda}(\PP,\QQ)
  &:= \sigma_{\mu,k,\lambda} \MMD_k(\PP,\QQ)
\\
  \sigma_{k,\mu,\lambda}
  &:= 1/\sqrt{\lambda + \int k(x,x)\mudx + \sum_{i=1}^d \int \partial_i \partial_{i+d} k(x,x) \mudx}
.\end{align*}
The Optimized SMMD over the restricted set $\Psi^{\kappa}$ is given by:
\[
  \optSMMD[\Psi^\kappa](\PP, \QQ) := \sup_{\psi\in \Psi^\kappa} \SMMD_{\mu,k_{\psi},\lambda}
.\]
The constraint to $\psi \in \Psi^\kappa$ is critical to the proof.
In practice, using a spectral parametrization helps enforce this assumption, as shown in \cref{fig:celebA_scores_and_singular_values,fig:singular_values-full}.
Other regularization methods, like orthogonal normalization \cite{Brock:2016}, are also possible.

We will use the following assumptions:
\begin{assumplist2}
  \item \label{full_support} $\mu$ is a probability distribution absolutely continuous with respect to the Lebesgue measure.
  \item \label{decreasing_dimensions} The dimensions of the weights are decreasing per layer: $d_{l+1}\leq d_{l}$ for all $0\leq l\leq L-1$.
  \item \label{leaky_relu} The non-linearity used is Leaky-ReLU, \eqref{eq:lReLU}, with leak coefficient $\alpha \in (0, 1)$.
  \item \label{Lichitz_kernel} The top-level kernel $K$ is globally Lipschitz in the RKHS norm: there exists a positive constant $L_K>0$ such that $\Vert K(a,.)-K(b,.) \Vert\leq L_K \Vert a-b \Vert $ for all $a$ and $b$ in $\R^{d_L}$.
  \item \label{Convexe_Hessian} There is some $\gamma_K > 0$ for which $K$ satisfies
  \begin{align}
    \nabla_b \nabla_c K(b, c) \bigr\vert_{(b,c) = (a,a)} \succeq \gamma^2 I \qquad \text{ for all } a \in \R^{d_L}.
  \end{align}
\end{assumplist2}

\cref{full_support}
ensures that the points where $\phi_\psi(X)$ is not differentiable
are reached with probability $0$ under $\mu$.
This assumption can be easily satisfied
e.g.\ if we define $\mu$ by adding Gaussian noise to $\PP$.

\Cref{decreasing_dimensions} helps ensure that the span of $W^l$ is never contained in the null space of $W^{l+1}$.
Using Leaky-ReLU as a non-linearity, \cref{leaky_relu},
further ensures that
the network $\phi_\psi$ is locally full-rank almost everywhere;
this might not be true with ReLU activations, where it could be always $0$.
\cref{decreasing_dimensions,leaky_relu} can be easily satisfied by design of the network.

\cref{Lichitz_kernel,Convexe_Hessian} only depend on the top-level kernel $K$ and are easy to satisfy in practice.
In particular, they always hold for a smooth translation-invariant kernel, such as the Gaussian,
as well as the linear kernel.

We are now ready to prove \cref{thm:continuity_opt_SMMD}.

\begin{reptheorem}{thm:continuity_opt_SMMD}
Under \cref{full_support,decreasing_dimensions,leaky_relu,Lichitz_kernel,Convexe_Hessian},
\[
  \optSMMD[\Psi^\kappa](\PP,\QQ) \leq \frac{L_K \, \kappa^{L/2}}{\gamma \, \sqrt{d_L} \, \alpha^{L/2}} \W(\PP, \QQ)
,\]
which implies that if $\PP_n \wassconv \PP$,
then $\optSMMD[\Psi^\kappa](\PP_n, \PP) \to 0$.
\end{reptheorem}
\begin{proof}

Define the pseudo-distance corresponding to the kernel $k_\psi$
\[
  d_\psi(x, y)
  = \lVert k_\psi(x, \cdot) - k_\psi(y, \cdot) \rVert_{\h_\psi}
  = \sqrt{k_\psi(x, x) + k_\psi(y, y) - 2 k_\psi(x, y)}
.\]
Denote by $\W_{d_{\psi}}(\PP,\QQ)$ the optimal transport metric
between $\PP$ and $\QQ$ using the cost $d_{\psi}$, given by
\[
  \W_{d_{\psi}}(\PP,\QQ)
  = \inf_{\pi\in\Pi(\PP,\QQ)} \E_{(X, Y) \sim \pi}\left[ d_{\psi}(X,Y) \right].
\]
where $\Pi$ is the set of couplings with marginals $\PP$ and $\QQ$.
By \cref{thm:mmd-w-upperbound},
\[
  \MMD_\psi(\PP, \QQ) \le \W_{d_\psi}(\PP, \QQ)
.\]
Recall that $\phi_\psi$ is Lipschitz,
$\lVert \phi_\psi \rVert_\lip < \infty$,
so along with \cref{Lichitz_kernel} we have that
\[
  d_\psi(x,y)
  \le L_K \norm{\phi_\psi(x) - \phi_\psi(y)}
  \le L_K \norm{\phi_\psi}_\lip \norm{x - y}
.\]
Thus
\[
  \W_{d_{\psi}}(\PP,\QQ)
  \le \inf_{\pi\in\Pi(\PP,\QQ)} \E_{(X, Y) \sim \pi}\left[ L_K\norm{\phi_\psi}_\lip \norm{X - Y} \right]
  = L_K \norm{\phi_\psi}_\lip \W(\PP, \QQ)
,\]
where $\W$ is the standard Wasserstein distance \eqref{eq:wasserstein},
and so
\[
  \MMD_\psi(\PP, \QQ) \le L_k \norm{\phi_\psi}_\lip \W(\PP, \QQ)
.\]

We have that
$\partial_i \partial_{i+d} k(x, y)
= \left[ \partial_i \phi_\psi(x) \right]\tp
  \left[ \nabla_a \nabla_b K(a, b) \bigr\vert_{(a,b)=(\phi_\psi(x), \phi_\psi(y))} \right]
  \left[ \partial_i \phi_\psi(y) \right]
,$
where the middle term is a $d_L \times d_L$ matrix
and the outer terms are vectors of length $d_L$.
Thus \cref{Convexe_Hessian} implies that
$\partial_i \partial_{i+d} k(x, x) \ge \gamma_K^2 \lVert \partial_i \phi_\psi(x) \rVert^2$,
and hence
\[
       \sigma_{\mu,k,\lambda}^{-2}
   \ge \gamma_K^2 \E[\lVert \nabla\phi_{\psi}(X)\rVert_F^2 ]
\]
so that
\[
  \SMMD^2_{\psi}(\PP, \QQ)
  = \sigma_{\mu,k,\lambda}^2 \MMD^2_\psi(\PP, \QQ)
  \le \frac{L_K^2 \norm{\phi_\psi}_\lip^2}{\gamma_K^2 \E\left[ \rVert \nabla \phi_\psi(X) \rVert_F^2 \right]} \W^2(\PP, \QQ)
.\]

Using \cref{appendix:prop:pseudo_homogeneity}, we can write
$\phi_{\psi}(X) = \alpha(\psi) \phi_{\bar{\psi}}(X)$
with $\bar{\psi} \in \Psi^{\kappa}_1$.
Then we have
\[
  \frac{ \norm{\phi_\psi}_\lip^2 }{\E_{\mu}\left[ \norm*{\nabla\phi_\psi(X)}_F^2 \right]}
  = \frac{\alpha(\psi)^2 \norm{\phi_{\bar\psi}}_\lip^2 }{\alpha(\psi)^2 \E_{\mu}\left[ \norm*{\nabla \phi_{\bar\psi}(X)}_F^2 \right]}
  \le \frac{1}{\E_{\mu}\left[ \norm*{\nabla \phi_{\bar\psi}(X)}_F^2 \right]}
,\]
where we used
$\norm{\phi_{\bar\psi}}_\lip \le \prod_{l=1}^L \norm{\bar W^l} = 1$.
But by \cref{lem:grad_essinf},
for Lebesgue-almost all $X$,
$\norm{\nabla \phi_{\bar\psi}(X)}_F^2 \ge d_L (\alpha / \kappa)^L$.
Using \cref{full_support},
this implies that
\[
  \frac{ \norm{\phi_\psi}_\lip^2 }{\E_{\mu}\left[ \norm*{\nabla\phi_\psi(X)}_F^2 \right]}
  \le \frac{1}{\E_\mu\left[ \norm{\nabla \phi_{\bar\psi}(X)}_F^2] \right]}
  \le \frac{\kappa^L}{d_L \alpha^L}
.\]

Thus for any $\psi \in \Psi^\kappa$,
\[
  \SMMD_\psi(\PP, \QQ) \le \frac{L_K \, \kappa^{L/2}}{\gamma_K \, \sqrt{d_L} \, \alpha^{L/2}} \W(\PP, \QQ)
.\]
The desired bound on $\optSMMD[\Psi^\kappa]$ follows immediately.
\end{proof}

\begin{lem} \label{thm:mmd-w-upperbound}
  Let $(x,y)\mapsto k(x,y)$ be the continuous kernel of an RKHS $\h$ defined on a Polish space $\mathcal{X}$, and define the corresponding pseudo-distance $d_k(x,y):= \norm{k(x,\cdot)-k(y,\cdot)}_\h$. Then the following  inequality holds for any distributions $\PP$ and $\QQ$ on $\mathcal{X}$, including when the quantities are infinite:
  \[
  \MMD_k(\PP,\QQ) \leq \W_{d_k}(\PP, \QQ)
  .\]
\end{lem}
\begin{proof}
Let $\PP$ and $\QQ$ be two probability distributions,
and let $\Pi(\PP, \QQ)$ be the set of couplings between them.
Let $\pi^* \in \argmin_{(X, Y) \sim \pi}[c_k(X, Y)]$ be an optimal coupling,
which is guaranteed to exist \citep[Theorem 4.1]{Villani:2009};
by definition $\W_{d_k}(\PP, \QQ) = \E_{(X, Y) \sim \pi^*}[ d_k(X, Y) ]$.
When $\W_{d_k}(\PP, \QQ) =\infty$ the inequality trivially holds, so assume that $\W_{d_k}(\PP, \QQ)<\infty$.

Take a sample $(X,Y) \sim \pi^{\star}$ and a function $f \in \h$ with  $\norm{f}_\h \leq 1$. By the Cauchy-Schwarz inequality,
\[
\norm{f(X) - f(Y)} \leq \norm{f}_\h \norm{k(X,\cdot)-k(Y,\cdot)}_\h \leq  \norm{k(X,\cdot)-k(Y,\cdot)}_\h
.\]

Taking the expectation with respect to $\pi^{\star}$, we obtain
\[
\E_{\pi^{\star}}[\vert f(X) - f(Y) \vert] \leq \E_{\pi^{\star}}[\norm{k(X,\cdot)-k(Y,\cdot)}_\h]
.\]
The right-hand side is just the definition of $\W_{d_k}(\PP,\QQ)$.
By Jensen's inequality, the left-hand side is lower-bounded by
\[
  \abs{ \E_{\pi^*}[ f(X) - f(Y) ]}
  = \abs{ \E_{X \sim \PP}[ f(X)] - \E_{Y \sim \QQ}[f(Y)] }
\] since $\pi^{\star}$ has marginals $\PP$ and $\QQ$.
We have shown so far that for any $f\in \h$ with $\Vert f \Vert_\h \leq 1$,
\[
\abs{ \E_{\PP}[ f(X)] - \E_{\QQ}[f(Y)] } \le \W_{c_k}(\PP,\QQ)
;\]
the result follows by taking the supremum over $f$.
\end{proof}

\begin{lem}
  \label{appendix:prop:pseudo_homogeneity}
  Let $\psi = ((W^L,b^L),(W^{L-1},b^{L-1}), \dots,(W^1,b^1)) \in \Psi^\kappa$.
  There exists a corresponding scalar $\alpha(\psi)$ and
  $\bar{\psi} = ((\bar{W}^L,\bar{b}^L), (\bar{W}^{L-1},\bar{b}^{L-1}), \dots, (\bar{W}^1,\bar{b}^1) ) \in \Psi^\kappa_1$,
  defined by \eqref{eq:def:psi-kappa-1},
  such that for all $X$,
  \[
    \phi_\psi(X) = \alpha(\psi) \operatorname{\phi_{\bar\psi}}(X)
  .\]
\end{lem}
\begin{proof}
Set $\bar{W}^l = \frac{1}{\norm{W^l}} W^l$, $\bar{b}^l = \frac{1}{\prod_{m=1}^l\norm{W^m}} b^l$,
and $\alpha(\psi) = \prod_{l=1}^L \norm{W^l}$.
Note that the condition number is unchanged,
$\cond(\bar W^l) = \cond(W^l) \le \kappa$,
and $\norm{\bar W^l} = 1$,
so $\bar\psi \in \Phi^\kappa_1$.
It is also easy to see from \eqref{eq:lReLU} that
\[
  h^l_{\bar \psi}(X) = \frac{1}{\prod_{m=1}^l \norm{W^m}} h^l_\psi(X)
\]
so that
\[
  \alpha(\psi) h^L_{\bar\psi}(X)
  = \frac{\prod_{l=1}^L \norm{W^l}}{\prod_{l=1}^L \norm{W^l}} h^L_\psi(X)
  = \phi_\psi(X)
.\qedhere\]
\end{proof}

\begin{lem} \label{lem:grad_essinf}
Make \cref{decreasing_dimensions,leaky_relu}, and let $\psi \in \Psi^{\kappa}_1$.
Then the set of inputs for which any intermediate activation is exactly zero,
\[
  \mathcal N_\psi = \bigcup_{l=1}^L \bigcup_{k=1}^{d_l} \left\{
    X \in \R^d \given \left( h_\psi^l(X) \right)_k = 0
  \right\}
,\]
has zero Lebesgue measure.
Moreover, for any $X \notin \N_\psi$,
$\nabla_X \phi_\psi(X)$ exists and
\[
  \norm{\nabla_X \phi_\psi(X)}_F^2
  \ge \frac{d_L \alpha^L}{\kappa^L}
.\]
\end{lem}
\begin{proof}
First, note that the network representation at layer $l$ is piecewise affine.
Specifically,
define $M_X^l \in \R^{d_l}$ by, using \cref{leaky_relu},
\[
  (M_X^l)_k = \sigma_l'(h^l_k(X))
  = \begin{cases}
    1 & h^l_k(X) > 0 \\
    \alpha & h^l_k(X) < 0 \\
  \end{cases}
;\]
it is undefined when any $h^l_k(X) = 0$,
i.e.\ when $X \in \N_\psi$.
Let $V^l_X := W^l \diag\left( M_X^{l-1} \right) $.
Then
\[
  h_\psi^l(X)
  = W^l \sigma_{l-1}(h_\psi^{l-1}(X)) + b^l
  = V^l_X X + b^l
,\]
and thus
\begin{equation}
  h_\psi^l(X)
  = \underbar{W}^l_X X + \underbar{b}^l_X
\label{eq:hl-affine}
,\end{equation}
where
$\underbar{b}^0_X = 0$, $\underbar{b}^l_X = V^l_X \underbar{b}^{l-1} + b^l$,
and $\underbar W^l_X = V^l_X V^{l-1}_X \cdots V^1_X$,
so long as $X \notin \N_\psi$.

Because $\psi \in \Psi_1^\kappa$,
we have $\norm{W^l} = 1$ and $\sigmamin(W^l) \ge 1/\kappa$;
also, $\norm{M_X^l} \le 1$, $\sigmamin(M_X^l) \ge \alpha$.
Thus $\norm{\underbar W^l_X} \le 1$,
and using \cref{decreasing_dimensions} with \cref{lem:min-sv}
gives $\sigmamin(\underbar W^l_X) \ge (\alpha / \kappa)^l$.
In particular, each $\underbar W^l_X$ is full-rank.

Next, note that $\underbar{b}^l_X$ and $\underbar{W}^l_X$
each only depend on $X$ through the activation patterns $M_X^l$.
Letting $H^l_X = (M_X^l, M_X^{l-1}, \dots, M_X^1)$ denote the full activation patterns up to level $l$,
we can thus write
\[
  h_\psi^l(X) = \underbar{W}^{H^l_X} X + \underbar{b}^{H^l_X}
.\]
There are only finitely many possible values for $H^l_X$;
we denote the set of such values as $\mathcal H^l$.
Then we have that
\[
             \mathcal N_\psi
   \subseteq \bigcup_{l=0}^L \bigcup_{k=1}^{d_L} \bigcup_{H^l \in \mathcal H^l} 
             \left\{
               X \in \R^d \given
               \underbar{W}_k^{H^l} X + \underbar{b}_k^{H^l} = 0
             \right\}
.\]
Because each $\underbar W_k^{H^l}$ is of rank $d_l$,
each set in the union is either empty
or an affine subspace of dimension $d - d_l$.
As each $d_l > 0$,
each set in the finite union has zero Lebesgue measure,
and $\mathcal N_\psi$ also has zero Lebesgue measure.

We will now show that the activation patterns are piecewise constant,
so that $\nabla_X h_\psi^l(X) = \underbar W^{H_X^l}$ for all $X \notin \N_\psi$.
Because $\psi \in \Psi^\kappa_1$,
we have $\norm{h_\psi^l}_\lip \le 1$,
and in particular
\[
  \abs*{
    \left( h_\psi^l(X) \right)_k
  - \left( h_\psi^l(X') \right)_k
  } \le \norm{X - X'}
.\]
Thus, take some $X \notin \N_\psi$,
and find the smallest absolute value of its activations,
$\epsilon = \min_{l = 1, \dots, L} \min_{k = 1, \dots, d_l} \abs*{\left(h^l_\psi(X)\right)_k}$;
clearly $\epsilon > 0$.
For any $X'$ with $\norm{X - X'} < \epsilon$,
we know that for all $l$ and $k$,
\[
  \operatorname{sign}\left( \left(h^l_\psi(X)\right)_k \right)
  = \operatorname{sign}\left( \left(h^l_\psi(X')\right)_k \right)
,\]
implying that
$H_X^l = H_{X'}^l$ as well as $X' \notin \N_\psi$.
Thus for any point $X \notin \N_\psi$,
$\nabla \phi_\psi(X) = \underbar{W}^{H_X^L}$.
Finally, we obtain
\[
  \norm{\nabla \phi_\psi(X)}_F^2
  = \norm{\underbar{W}^{H_X^L}}_F^2
  \ge d_L \sigmamin\left(\underbar{W}^{H_X^L}\right)^2
  \ge \frac{d_L \alpha^L}{\kappa^L}
.\qedhere\]
\end{proof}

\begin{lem} \label{lem:min-sv}
  Let $A \in \R^{m \times n}$, $B \in \R^{n \times p}$, with $m \ge n \ge p$. Then
  $
    \sigmamin(A B) \ge \sigmamin(A) \sigmamin(B)
  $.
\end{lem}
\begin{proof}
A more general version of this result can be found in \cite[Theorem 2]{Gungor:2007}; we provide a proof here for completeness.

  If $B$ has a nontrivial null space, $\sigmamin(B) = 0$ and the inequality holds.
  Otherwise, let $\R^n_*$ denote $\R^n \setminus \{0\}$.
  Recall that for $C \in \R^{m \times n}$ with $m \ge n$,
  \[
    \sigmamin(C)
    = \sqrt{\lambda_{\min}(C\tp C)}
    = \sqrt{\inf_{x \in \R^n_*} \frac{x\tp C\tp C x}{x\tp x}}
    = \inf_{x \in \R^n_*} \frac{\norm{C x}}{\norm{x}}
  .\]
  Thus, as $B x \ne 0$ for $x \ne 0$,
  \begin{align*}
         \sigmamin(A B)
    &  = \inf_{x \in \R^p_*} \frac{\norm{A B x}}{\norm x}
       = \inf_{x \in \R^p_*} \frac{\norm{A B x} \norm{B x}}{\norm{B x} \norm{x}}
  \\&\ge \left( \inf_{x \in \R^p_*} \frac{\norm{A B x}}{\norm{B x}} \right)
         \left( \inf_{x \in \R^p_*} \frac{\norm{B x}}{\norm{x}} \right)
  \\&\ge \left( \inf_{y \in \R^n_*} \frac{\norm{A y}}{\norm{y}} \right)
         \left( \inf_{x \in \R^p_*} \frac{\norm{B x}}{\norm{x}} \right)
       = \sigmamin(A) \sigmamin(B)
  . \qedhere\end{align*}
\end{proof}

\subsubsection{When some of the assumptions don't hold} \label{sec:counterexamples}
Here we analyze through simple examples what happens when the condition number can be unbounded, and when \cref{decreasing_dimensions}, about decreasing widths of the network, is violated.
\paragraph{Condition Number:}\label{example_condition_number}
We start by a first example where the condition number can be arbitrarily high. We consider a two-layer network on $\R^2$, defined by
\begin{align}
  \phi_{\alpha}(X)= \begin{bmatrix}1 & -1\end{bmatrix} \sigma(W_\alpha X) \qquad W_{\alpha} = \begin{bmatrix}
    1 & 1\\
    1 & 1+\alpha
  \end{bmatrix}
\label{eq:example_1}
\end{align}
where $\alpha >0$. As $\alpha$ approaches $0$ the matrix $W_\alpha$ becomes singular which means that its condition number blows up. We are interested in analyzing the behavior of the Lipschitz constant of $\phi$ and the expected squared norm of its gradient under $\mu$ as $\alpha$ approaches $0$.

One can easily compute the squared norm of the gradient of $\phi$ which is given by
\begin{align}
  \Vert \nabla \phi_{\alpha}(X)\Vert^2 = \begin{cases}
    \alpha^2 & X\in A_1\\
    \gamma^2 \alpha^2 & X\in A_2\\
    (1-\gamma)^2 + (1+\alpha -\gamma)^2 & X\in A_3\\
    (1-\gamma)^2 + (\gamma \alpha +\gamma -1)^2 & X\in A_4
  \end{cases}
\end{align}
Here $A_1$, $A_2$, $A_3$ and $A_4$ are defined by \cref{eq:sets_examples_1} and are represented in \cref{fig:example_1_domains}:
\begin{align}\label{eq:sets_examples_1}
\begin{split}
A_1 &:= \{X\in\R^2 | X_1 + X_2\geq 0 \quad X_1 + (1+\alpha)X_2\geq 0 \}\\
A_2 &:= \{X\in\R^2 | X_1 + X_2< 0 \quad X_1 + (1+\alpha)X_2< 0 \}\\
A_3 &:= \{X\in\R^2 | X_1 + X_2< 0 \quad X_1 + (1+\alpha)X_2\geq 0 \}\\
A_4 &:= \{X\in\R^2 | X_1 + X_2\geq 0 \quad X_1 + (1+\alpha)X_2< 0 \}
\end{split}
\end{align}

\begin{figure}[ht]
\centering
        \includegraphics[width=0.7\linewidth]{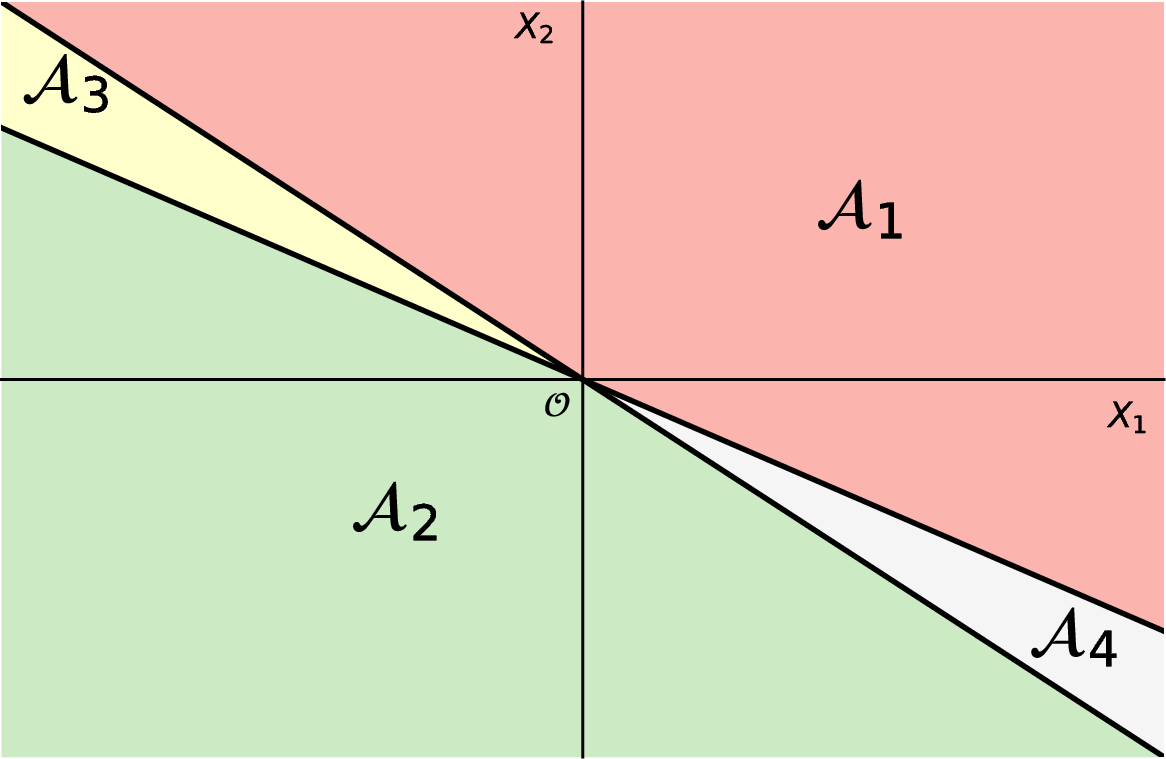}
            \caption{Decomposition of $\R^2$ into 4 regions $A_1$, $A_2$, $A_3$ and $A_4$ as defined in \cref{eq:sets_examples_1}. As $\alpha$ approaches $0$, the area of sets $A_3$ and $A_4$ becomes negligible.}
    \label{fig:example_1_domains}
\end{figure}

It is easy to see that whenever $\mu$ has a density, the probability of the sets $A_3$ and $A_4$ goes to $0$ are $\alpha \rightarrow 0$. Hence one can deduce that $\E_{\mu}[\Vert \nabla \phi_{\alpha}(X)\Vert^2] \rightarrow  0$ when $\alpha \rightarrow 0$. On the other hand, the squared Lipschitz constant of $\phi$ is given by $(1-\gamma)^2 + (1+\alpha -\gamma)^2 $ which converges to $2(1-\gamma)^2$. This shows that controlling the expectation of the gradient doesn't allow to effectively control the Lipschitz constant of $\phi$.

\paragraph{Monotonicity of the dimensions:}\label{example_monotonicity_dimensions} We would like to consider a second example where \cref{decreasing_dimensions} doesn't hold. Consider the following two layer network defined by:
\begin{align}
  \phi(X) = \begin{bmatrix}-1 & 0 & 1\end{bmatrix} \sigma(W_{\beta}X)
  \qquad
  W_{\beta} := \begin{bmatrix}
    1 & 0\\
    0 & 1\\
    1 & \beta
  \end{bmatrix}
\end{align}
for $\beta >0$. Note that $W_{\beta}$ is a full rank matrix, but \cref{decreasing_dimensions} doesn't hold. Depending on the sign of the components of $W_{\beta}X$ one has the following expression for $\Vert \nabla \phi_{\alpha}(X)\Vert^2 $:
\begin{align}
  \Vert \nabla \phi_{\alpha}(X)\Vert^2 = \begin{cases}
    \beta^2 & X\in B_1\\
    \gamma^2 \beta^2 & X\in B_2\\
    \beta^2 & X\in B_3\\
    (1-\gamma)^2 + \gamma^2\beta^2 & X\in B_4\\
    (1-\gamma)^2 + \beta^2 & X\in B_5\\
    \gamma^2 \beta^2 & X\in B_6\\
  \end{cases}
\end{align}
where $(B_i)_{1\leq i\leq 6}$ are defined by \cref{eq:sets_examples_2}
\begin{align}\label{eq:sets_examples_2}
\begin{split}
B_1 &:= \{X\in\R^2 | X_1 \geq 0 \quad X_2 \geq 0 \}\\
B_2 &:= \{X\in\R^2 | X_1 < 0 \quad X_2 < 0 \}\\
B_3 &:= \{X\in\R^2 | X_1 \geq \quad X_2 < 0 \quad X_1 +\beta X_2 \geq 0  \}\\
B_4 &:= \{X\in\R^2 | X_1 \geq \quad X_2 < 0 \quad X_1 +\beta X_2 < 0 \}\\
B_5 &:= \{X\in\R^2 | X_1 > 0 \quad X_2 \geq 0 \quad X_1 +\beta X_2 \geq 0  \}\\
B_6 &:= \{X\in\R^2 | X_1 > 0 \quad X_2 \geq 0 \quad X_1 +\beta X_2 < 0  \}
\end{split}
\end{align}

The squared Lipschitz constant is given by $\Vert \phi \Vert_{L}^2 (1-\gamma)^2+\beta^2$ while the expected squared norm of the gradient of $\phi$ is given by:
\begin{align}
  \E_{\mu}[\Vert \phi(X) \Vert^2] = 3\beta^2(p(B_1\cup B_3\cup B_5)+ \gamma^2 p(B_2\cup B_4\cup B_6)) + (1-\gamma)^2p(B_4\cup B_5).
\end{align}
Again the set $B_4\cup B_5$ becomes negligible as $\beta$ approaches $0$ which implies that  $\E_{\mu}[\Vert \phi(X) \Vert^2] \rightarrow 0$. On the other hand $\Vert \phi \Vert_{L}^2 $ converges to $(1-\gamma)^2$.  Note that unlike in the first example in \cref{eq:example_1}, the matrix $W_{\beta}$ has a bounded condition number. In this example, the columns of $W_{0}$  are all in the null space of $\begin{bmatrix}-1 & 0 & 1\end{bmatrix}$, which implies $\nabla\phi_0(X)=0$ for all $X\in \R^2$, even though all matrices have full rank.

\section{DiracGAN vector fields for more losses} \label{appendix:diracgan-full}
\begin{figure}[ht]
\centering
   \includegraphics[width=\linewidth]{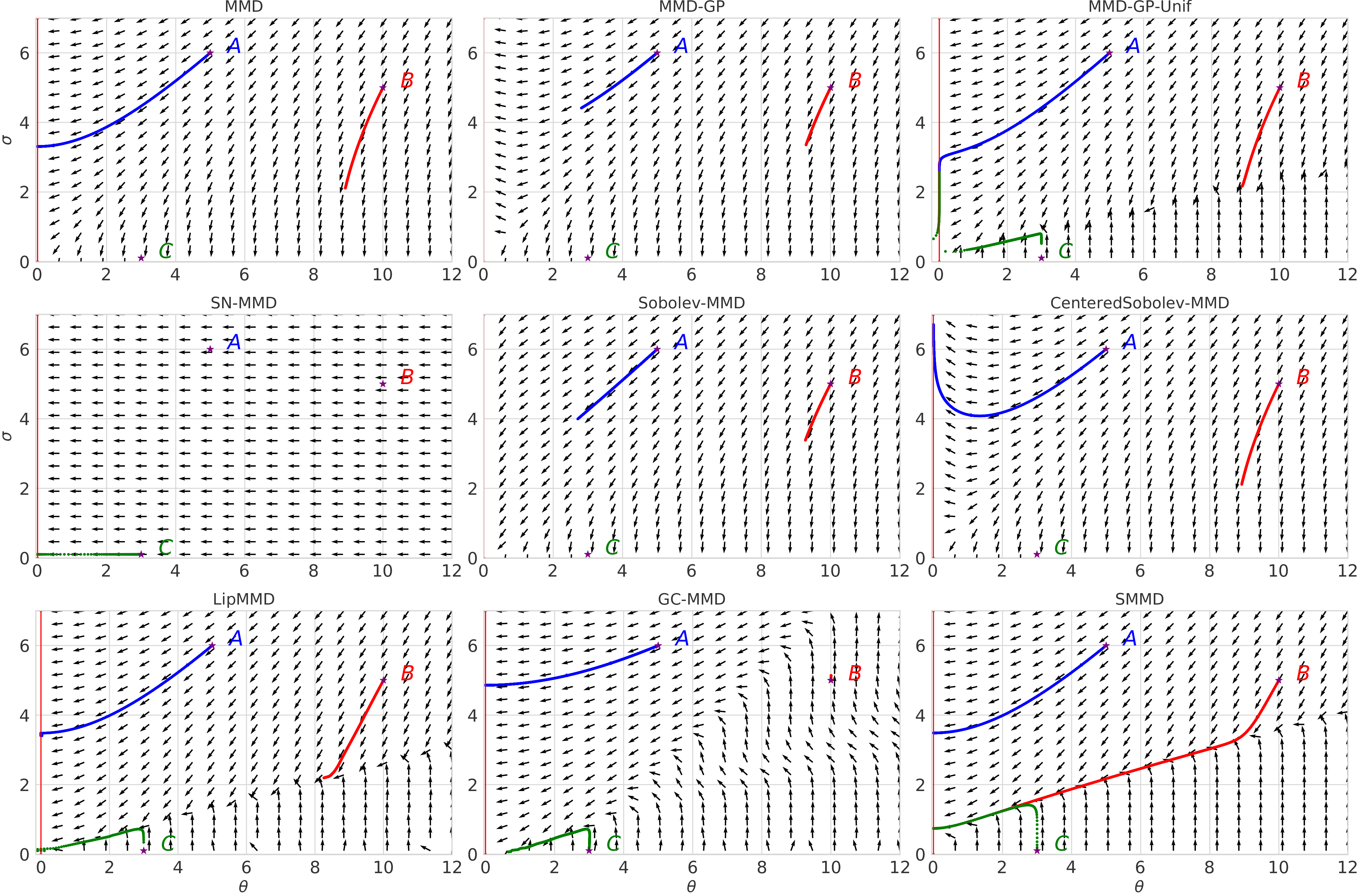}
    \caption{Vector fields for different losses with respect to the generator parameter $\theta$ and the feature representation parameter $\psi$;
    the losses use a Gaussian kernel, and are shown in \eqref{eq:losses_dirac_gan}.
    Following \cite{Mescheder:2018}, $\PP = \delta_0$, $\QQ=\delta_{\theta}$ and $\phi_{\psi}(x)= \psi x$.
    The curves show the result of taking simultaneous gradient steps in $(\theta, \psi)$
    beginning from three initial parameter values.}
    \label{fig:vector_fields_all}
\end{figure}
\Cref{fig:vector_fields_all} shows parameter vector fields, like those in \cref{fig:critic_vector_fields}, for \cref{example:diracgan}
for a variety of different losses:
\begin{align}\label{eq:losses_dirac_gan}
\begin{split}
  \text{MMD: }& -\MMD_\psi^2\\
  \text{MMD-GP: }&  -\MMD_\psi^2 +\lambda \E_{\PP}[(\Vert \nabla f(X) \Vert -1)^2]\\
  \text{MMD-GP-Unif: }&  -\MMD_\psi^2 +\lambda \E_{\widetilde{X}\simeq\mu^{*}}[(\Vert \nabla f(\widetilde{X}) \Vert -1)^2]\\
  \text{SN-MMD: }&  -2 \MMD_1(\PP, \QQ)^2 \\
  \text{Sobolev-MMD: }&  -\MMD_\psi^2 + \lambda (\E_{(\PP+\QQ)/2}[\Vert \nabla f(X) \Vert^2]-1)^2\\
  \text{CenteredSobolev-MMD: }&  -\MMD_\psi^2 + \lambda (\E_{(\PP+\QQ)/2}[\Vert \nabla f(X) \Vert^2])^2\\
  \text{LipMMD: }&  -\LipMMD_{k_\psi, \lambda}^2 \\
  \text{GC-MMD: }&  -\GCMMD_{\mathcal{N}(0, 10^2), k_\psi, \lambda}^2 \\
  \text{SMMD: }& -\SMMD_{k_\psi, \PP, \lambda}^2
  \end{split}
\end{align}
The squared MMD between $\delta_0$ and $\delta_{\theta}$ under a Gaussian kernel of bandwidth $1/\psi$ and is given by $2(1-e^{-\frac{\psi^2\theta^2}{2}})$.
MMD-GP-unif uses a gradient penalty as in \cite{Binkowski:2018} where each samples from $\mu^{*}$ is obtained by first sampling $X$ and $Y$ from $\PP$ and $\QQ$ and then sampling uniformly between $X$ and $Y$.
MMD-GP uses the same gradient penalty, but the expectation is taken under $\PP$ rather than $\mu^{*}$.
SN-MMD refers to MMD with spectral normalization; here this means that $\psi=1$.
Sobolev-MMD refers to the loss used in \cite{sobolev-gan} with the quadratic penalty only.
$\GCMMD_{\mu,k,\lambda}$ is defined by \cref{eq:SobolevMMD}, with $\mu = \N(0, 10^2)$.

\section{Vector fields of Gradient-Constrained MMD and Sobolev GAN critics} \label{appendix:critic-vector-fields}
\citet{sobolev-gan} argue that \emph{the gradient of the critic (...) defines a transportation plan
for moving the distribution mass} (from generated to reference distribution) and present
the solution of Sobolev PDE for 2-dimensional Gaussians. We observed that in this simple example
the gradient of the Sobolev critic can be very high outside of the areas of high density,
which is not the case with the Gradient-Constrained MMD.
\Cref{fig:critic_vector_fields}
presents critic gradients in both cases,
using $\mu = (\PP + \QQ) / 2$ for both.
\begin{figure}[ht]
\centering
    \begin{subfigure}[t]{.48\linewidth}
        \centering
        \includegraphics[width=\linewidth]{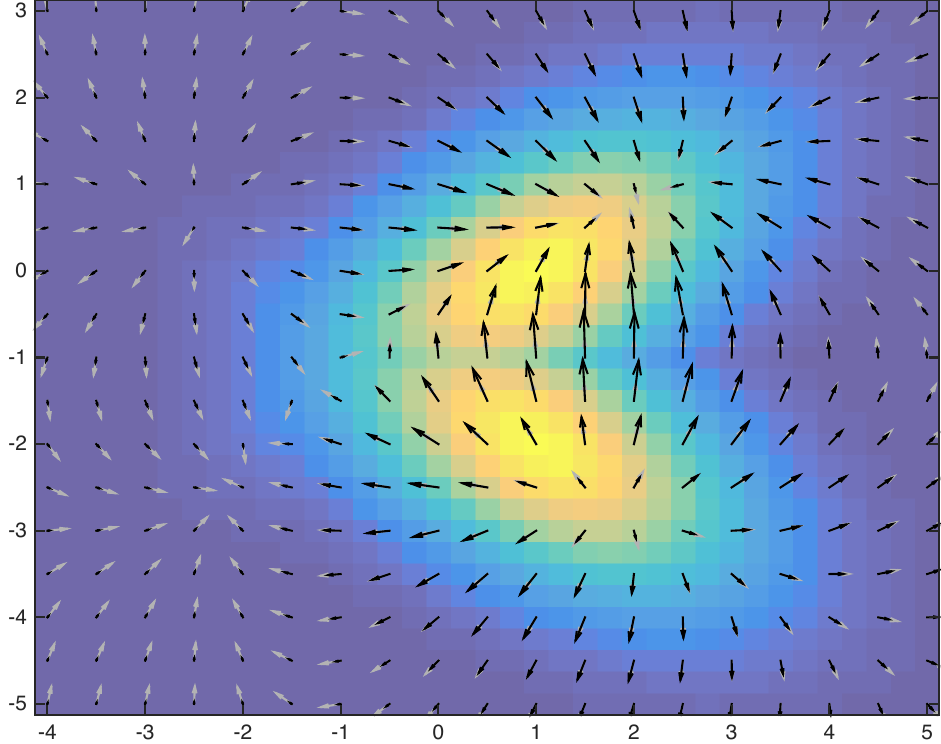}
        \caption{Gradient-Constrained MMD critic gradient.}
    \end{subfigure}
    ~
    \begin{subfigure}[t]{.48\linewidth}
        \centering
        \includegraphics[width=\linewidth]{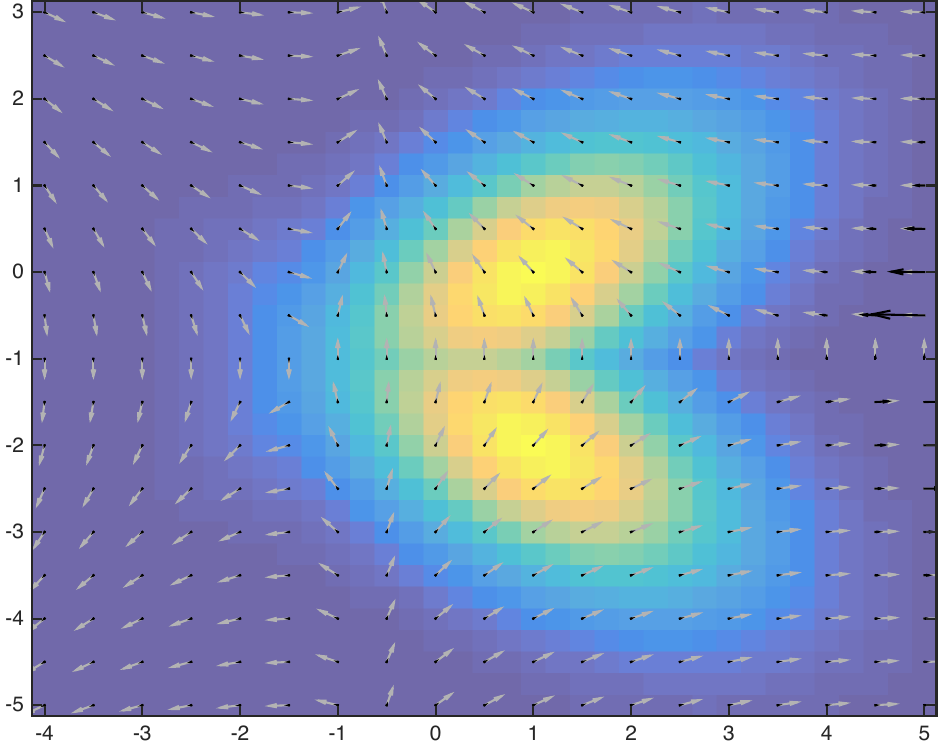}
        \caption{Sobolev IPM critic gradient.}
    \end{subfigure}
    \caption{Vector fields of critic gradients between two Gaussians. The grey arrows show
normalized gradients, i.e. gradient directions, while the black ones are the actual gradients.
Note that for the Sobolev critic, gradients norms are orders of magnitudes higher on the right hand side
of the plot than in the areas of high density of the given distributions.}
    \label{fig:critic_vector_fields}
\end{figure}

This unintuitive behavior is most likely
related to the vanishing boundary condition, assummed by Sobolev GAN. Solving the actual Sobolev PDE,
we found that the Sobolev critic has very high gradients close to the boundary in order to match
the condition; moreover, these gradients point in opposite directions to the target distribution.

\section{An estimator for Lipschitz MMD} \label{sec:est-lipmmd}
We now describe briefly how to estimate the Lipschitz MMD in low dimensions.
Recall that
\[
  \LipMMD_{k,\lambda}(\PP, \QQ)
  = \sup_{f \in \h_k \,:\, \norm{f}_\lip^2 + \lambda \norm{f}_{\h_k}^2 \le 1}
    \E_{X \sim \PP}[f(X)] - \E_{X \sim \QQ}[f(Y)]
.\]
For $f \in \h_k$,
it is the case that
\[
  \norm{f}_\lip^2
  = \sup_{x \in \R^d} \norm{\nabla f(x)}^2
  = \sup_{x \in \R^d} \sum_{i=1}^d \langle \partial_i k(x, \cdot), f \rangle_{\h_k}^2
  = \sup_{x \in \R^d} \left\langle f, \sum_{i=1}^d \left[ \partial_i k(x, \cdot) \otimes \partial_i k(x, \cdot) \right] f \right\rangle_{\h_k}
.\]
Thus we can approximate the constraint
$\norm{f}_\lip^2 + \lambda \norm{f}_{\h_k}^2 \le 1$
by enforcing the constraint on a set of $m$ points $\{ Z_i \}$
reasonably densely covering the region around the supports of $\PP$ and $\QQ$,
rather than enforcing it at every point in $\mathcal X$.
An estimator of the Lipschitz MMD
based on $X \sim \PP^{n_X}$ and $Y \sim \QQ^{n_Y}$ is
\begin{equation}
  \begin{split}
  \widehat\LipMMD_{k,\lambda}\left(
    X, Y, Z
  \right)
  \approx
  \sup_{f \in \h_k}&\,
    \frac{1}{n_X} \sum_{j=1}^{n_X} f(X_j)
  - \frac{1}{n_Y} \sum_{j=1}^{n_Y} f(Y_j)
  \\
  \text{s.t. }&
  \forall j, \;
  \norm{\nabla f(Z_j)}^2 + \lambda \norm{f}_{\h_k}^2 \le 1
  .\end{split}
  \label{eq:lipmmd-approx-problem}
\end{equation}
By the generalized representer theorem,
the optimal $f$ for \eqref{eq:lipmmd-approx-problem} will be of the form
\[
  f(\cdot)
  = \sum_{j=1}^{n_X} \alpha_j k(X_j, \cdot)
  + \sum_{j=1}^{n_Y} \beta_j k(Y_j, \cdot)
  + \sum_{i=1}^d \sum_{j=1}^{m} \gamma_{(i,j)} \partial_i k(Z_j, \cdot)
.\]
Writing $\delta = \left(\alpha, \beta, \gamma\right)$,
the objective function is linear in $\delta$,
\[\begin{bmatrix}
  \frac{1}{n_X} & \cdots & \frac{1}{n_X} &
  -\frac{1}{n_Y} & \cdots & -\frac{1}{n_Y} &
  0 & \cdots & 0
\end{bmatrix} \delta
.\]
The constraints are quadratic, built from the following matrices,
where the $X$ and $Y$ samples are concatenated together,
as are the derivatives with each dimension of the $Z$ samples:
\begin{align*}
  K &:= \begin{bmatrix}
    k(X_1, X_1) & \cdots & k(X_1, Y_{n_Y}) \\
    \vdots & \ddots & \vdots \\
    k(Y_{n_X}, X_1) & \cdots & k(Y_{n_Y}, Y_{n_Y}) \\
  \end{bmatrix}
  \\
  B &:= \begin{bmatrix}
    \partial_1 k(Z_1, X_1) & \cdots & \partial_1 k(Z_1, Y_{n_Y}) \\
    \vdots & \ddots & \vdots \\
    \partial_d k(Z_m, X_1) & \cdots & \partial_d k(Z_m, Y_{n_Y}) \\
  \end{bmatrix}
  \\
  H &:= \begin{bmatrix}
    \partial_1 \partial_{1+d} k(Z_1, Z_1) & \cdots & \partial_1 \partial_{d+d} k(Z_1, Z_m) \\
    \vdots & \ddots & \vdots \\
    \partial_d \partial_{1+d} k(Z_m, Z_1) & \cdots & \partial_d \partial_{d+d} k(Z_m, Z_m) \\
  \end{bmatrix}
.\end{align*}
Given these matrices, and letting
$O_j = \sum_{i=1}^d e_{(i,j)} e_{(i,j)}\tp$
where $e_{(i,j)}$ is the $(i,j)$th standard basis vector in $\R^{m d}$,
we have that
\[
  \norm{f}_{\h_k}^2
  = \delta\tp
  \begin{bmatrix}
    K & B\tp \\ B & H
  \end{bmatrix}
  \delta
  \qquad
  \norm{\nabla f(Z_j)}^2
  = \sum_{i=1}^d \left( \partial_i f(Z_j) \right)^2
  = \delta\tp
  \begin{bmatrix}
      B\tp O_j B
    & B\tp O_j H
   \\  H   O_j B
    &  H   O_j H
  \end{bmatrix}
  \delta
.\]
Thus the optimization problem \eqref{eq:lipmmd-approx-problem}
is a linear problem with convex quadratic constraints,
which can be solved by standard convex optimization software.
The approximation is reasonable only if we can effectively cover the region of interest with densely spaced $\{Z_i\}$;
it requires a nontrivial amount of computation even for the very simple 1-dimensional toy problem of \cref{example:diracgan}.

One advantage of this estimator, though,
is that finding its derivative with respect to the input points or the kernel parameterization
is almost free once we have computed the estimate,
as long as our solver has computed the dual variables $\mu$ corresponding to the constraints in \eqref{eq:lipmmd-approx-problem}.
We just need to exploit the envelope theorem and then differentiate the KKT conditions,
as done for instance in \cite{Amos2017}.
The differential of \eqref{eq:lipmmd-approx-problem} ends up being,
assuming the optimum of \eqref{eq:lipmmd-approx-problem} is at $\hat\delta \in \R^{n_X + n_Y + m d}$ and $\hat\mu \in \R^m$,
\begin{gather*}
     \ud\widehat\LipMMD_{k,\lambda}(X, Y, Z)
   = \hat\delta\tp
     \begin{bmatrix} \ud K \\ \ud B \end{bmatrix}
     \begin{bmatrix} \frac{1}{n_X} & \cdots & \frac{1}{n_X} & -\frac{1}{n_Y} & \cdots & -\frac{1}{n_Y} \end{bmatrix}\tp
  - \sum_{j=1}^m \hat\mu_j \hat\delta\tp (\ud P_j) \hat\delta
\\   \ud P_j
   := \begin{bmatrix}
        (\ud B)\tp O_j B + B\tp O_j (\ud H)
      & (\ud B)\tp O_j H + B\tp O_j (\ud H)
     \\ (\ud H) O_j B + H O_j (\ud B)
      & (\ud H) O_j H + H O_j (\ud H)
    \end{bmatrix}
    + \lambda \begin{bmatrix}
      \ud K & \ud B\tp \\
      \ud B & \ud H
     \end{bmatrix}
.\end{gather*}

\section{Near-equivalence of WGAN and linear-kernel MMD GANs} \label{appendix:wgan-linear-kernel}
For an MMD GAN-GP with kernel $k(x, y) = \phi(x) \phi(y)$,
we have that
\[
  \MMD_k(\PP, \QQ)
  = \lvert \E_\PP \phi(x) - \E_\QQ \phi(Y) \rvert
\]
and the corresponding critic function is
\[
  \frac{\eta(t)}{\lVert \eta \rVert_\h}
  = \frac{\E_{X \sim \PP} \phi(X) \phi(t) - \E_{Y \sim \QQ} \phi(Y) \phi(t)}{\lvert \E_\PP \phi(X) - \E_\QQ \phi(Y) \rvert}
  = \operatorname{sign}\left( \E_{X \sim \PP} \phi(X) - \E_{Y \sim \QQ} \phi(Y) \right)
    \phi(t)
.\]
Thus if we assume $\E_{X \sim \PP} \phi(X) > \E_{Y \sim \QQ} \phi(Y)$,
as that is the goal of our critic training,
we see that the MMD becomes identical to the WGAN loss,
and the gradient penalty is applied to the same function.

(MMD GANs, however, would typically train on the unbiased estimator of $\MMD^2$,
giving a very slightly different loss function.
\cite{Binkowski:2018} also applied the gradient penalty to $\eta$ rather than the true critic $\eta / \lVert \eta \rVert$.)

The SMMD with a linear kernel is thus analogous to applying the scaling operator to a WGAN;
hence the name SWGAN.

\section{Additional experiments}

\subsection{Comparison of Gradient-Constrained MMD to Scaled MMD} \label{appendix:gcmmd-smmd}

\Cref{fig:isolines} shows the behavior of the MMD, the Gradient-Constrained SMMD, and the Scaled MMD when comparing Gaussian distributions.
We can see that $\MMD \propto \SMMD$ and the Gradient-Constrained MMD behave similarly in this case,
and that optimizing the $\SMMD$ and the Gradient-Constrained MMD is also similar.
Optimizing the MMD would yield an essentially constant distance.

\begin{figure}[p]
  \centering
  \includegraphics[width=\linewidth]{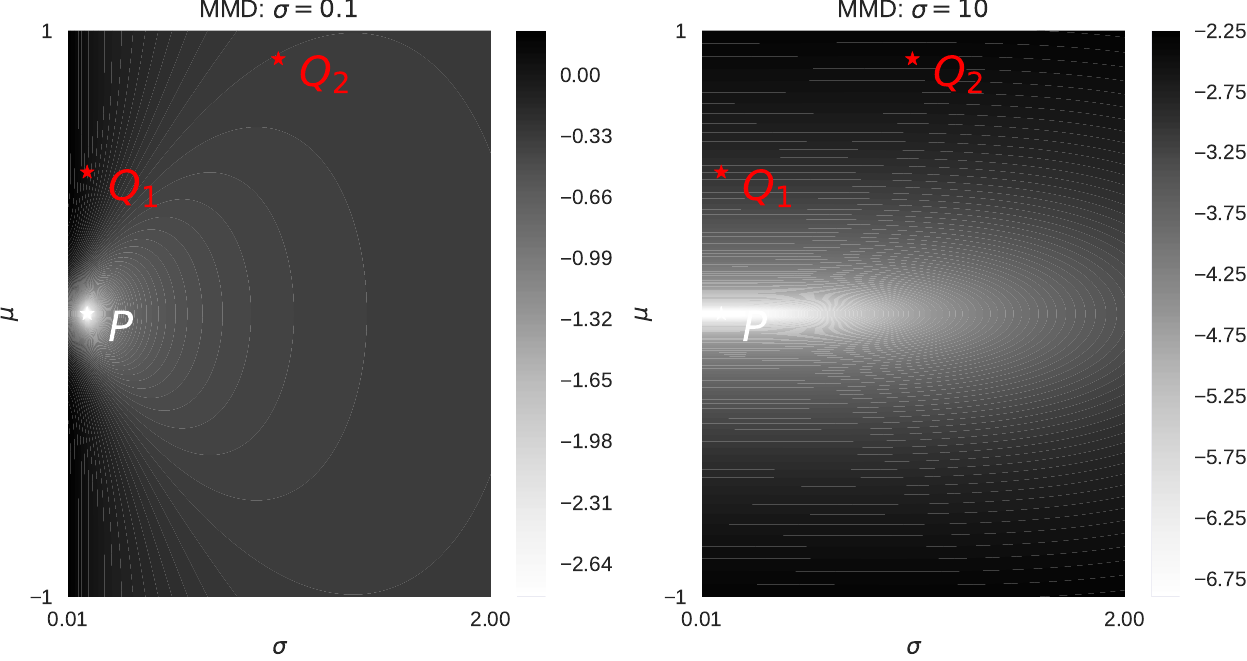}
  \includegraphics[width=\linewidth]{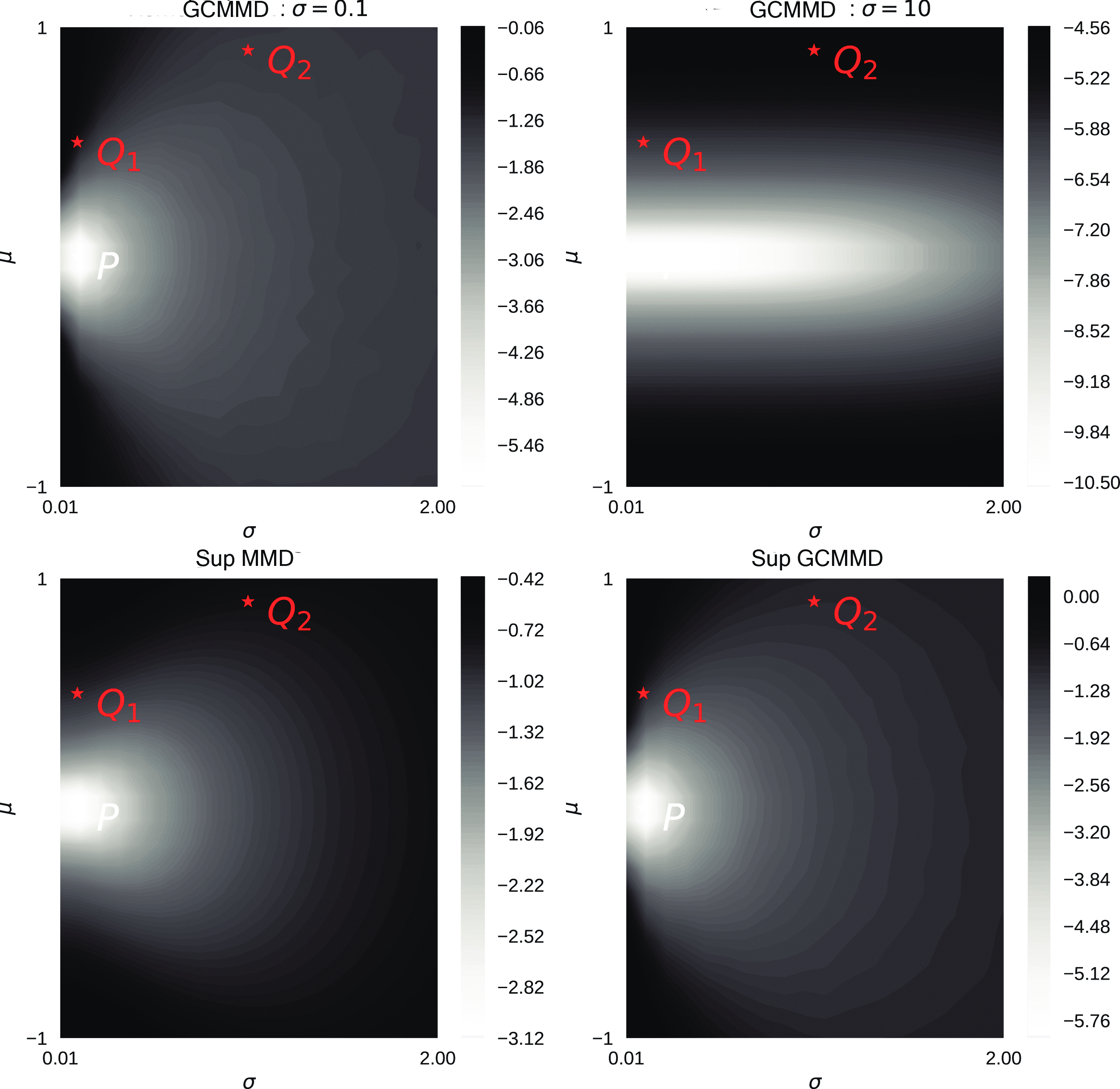}
  \caption{%
    Plots of various distances between one dimensional Gaussians,
    where $P = \N(0, 0.1^2)$,
    and the colors show $\log \D(P, \N(\mu, \sigma^2))$.
    All distances use $\lambda = 1$.
    Top left: MMD with a Gaussian kernel of bandwidth $\psi = 0.1$.
    Top right: MMD with bandwidth $\psi = 10$.
    Middle left: Gradient-Constrained MMD with bandwidth $\psi = 0.1$.
    Middle right: Gradient-Constrained MMD with bandwidth $\psi = 10$.
    Bottom left: Optimized SSMD, allowing any $\psi \in \R$.
    Bottom right: Optimized Gradient-Constrained MMD.
  }
  \label{fig:isolines}
\end{figure}

\subsection{IGMs with Optimized Gradient-Constrained MMD loss} \label{appendix:sobolev-expt}

We implemented the estimator of \cref{prop:Finite_rank_approx}
using the empirical mean estimator of $\eta$,
and sharing samples for $\mu = \PP$.
To handle the large but approximately low-rank matrix system,
we used an incomplete Cholesky decomposition \parencite[Algorithm 5.12]{shawe-taylor-christianini}
to obtain $R \in \R^{\ell \times M (1 + d)}$
such that $\begin{bmatrix} K & G\tp \\ G & H \end{bmatrix} \approx R\tp R$.
Then the Woodbury matrix identity allows an efficient evaluation:
\[
  \left( R\tp R + M \lambda I \right)^{-1}
  = \frac{1}{M \lambda} \left( I - R (R R\tp + M \lambda I)^{-1} R \right)
.\]
Even though only a small $\ell$ is required for a good approximation,
and the full matrices $K$, $G$, and $H$ need never be constructed,
backpropagation through this procedure is slow
and not especially GPU-friendly; training on CPU was faster.
Thus we were only able to run the estimator on MNIST,
and even that took days to conduct the optimization on powerful workstations.

The learned models, however, were reasonable.
Using a DCGAN architecture,
batches of size 64,
and a procedure that otherwise agreed with the setup of \cref{sec:experiments},
samples with and without spectral normalization
are shown in \cref{fig:mnist:sobolev:nosn,fig:mnist:sobolev:sn}.
After the points in training shown, however,
the same rank collapse as discussed in \cref{sec:experiments} occurred.
Here it seems that spectral normalization may have delayed the collapse,
but not prevented it.
\Cref{fig:mnist:loss} shows generator loss estimates through training,
including the obvious peak at collapse;
\cref{fig:mnist:kid} shows KID scores based on the MNIST-trained convnet representation \citep{Binkowski:2018},
including comparable SMMD models for context.
The fact that SMMD models converged somewhat faster than Gradient-Constrained MMD models here
may be more related to properties of
the estimator of \cref{prop:Finite_rank_approx}
rather than the distances;
more work would be needed to fully compare the behavior of the two distances.

\begin{figure}
  \begin{subfigure}{.2\linewidth}
    \includegraphics[width=\linewidth]{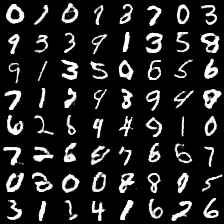}
    \caption{Without spectral normalization; $32\,000$ generator iterations.}
    \label{fig:mnist:sobolev:nosn}
  \end{subfigure}
  \begin{subfigure}{.2\linewidth}
    \includegraphics[width=\linewidth]{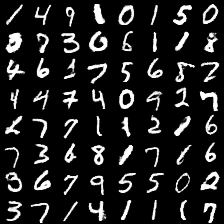}
    \caption{With spectral normalization; $41\,000$ generator iterations.}
    \label{fig:mnist:sobolev:sn}
  \end{subfigure}
  \begin{subfigure}{.28\linewidth}
    \includegraphics[width=\linewidth]{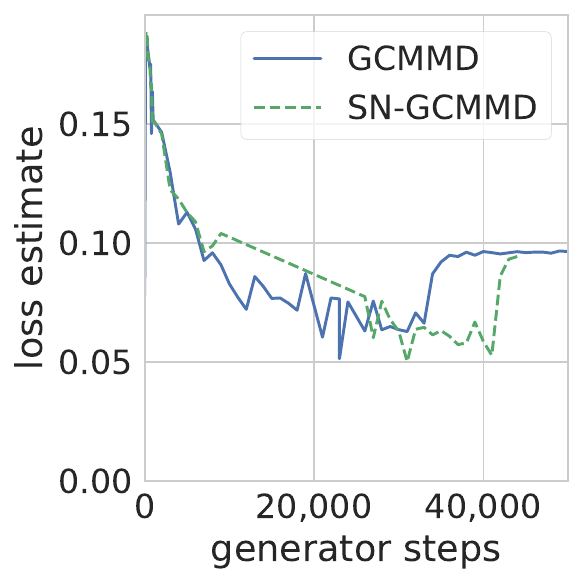}
    \caption{Generator losses.}
    \label{fig:mnist:loss}
  \end{subfigure}
  \begin{subfigure}{.28\linewidth}
    \includegraphics[width=\linewidth]{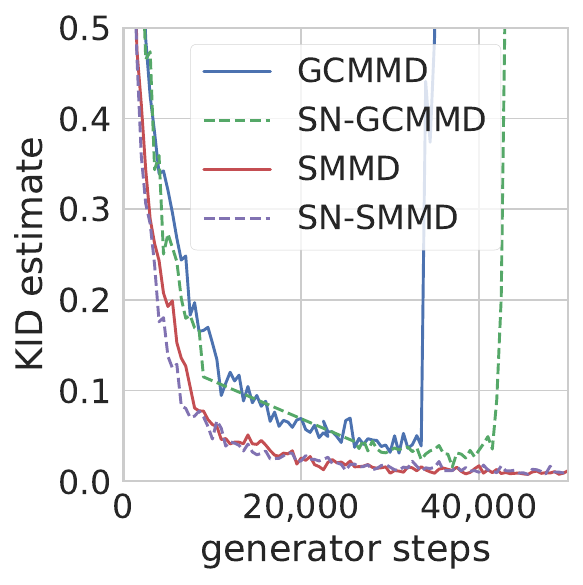}
    \caption{KID scores.}
    \label{fig:mnist:kid}
  \end{subfigure}
  \caption{The MNIST models with Optimized Gradient-Constrained MMD loss.}
\end{figure}

\subsection{Spectral normalization and Scaled MMD} \label{appendix:smmd_vs_sn}
\Cref{fig:singular_values-full} shows the distribution of critic weight singular values, like \cref{fig:celebA_scores_and_singular_values}, at more layers.
\Cref{fig:score_per_iter_cifar10_sn,tab:sn_cifar10_scores}
show results for the spectral normalization variants considered in the experiments.
MMDGAN, with neither spectral normalization nor a gradient penalty, did surprisingly well in this case, though it fails badly in other situations.

\cref{fig:singular_values-full} compares the decay of singular values for layer of the critic's network at both early and later stages of training in two cases: with or without the spectral parametrization. The model was trained on CelebA using SMMD.
\cref{fig:score_per_iter_cifar10_sn} shows the evolution per iteration of Inception score, FID and KID for Sobolev-GAN, MMDGAN and variants of MMDGAN and WGAN using spectral normalization. It is often the case that this parametrization alone is not enough to achieve good results.

\begin{figure}[ht]

        \centering
        \includegraphics[width=\linewidth]{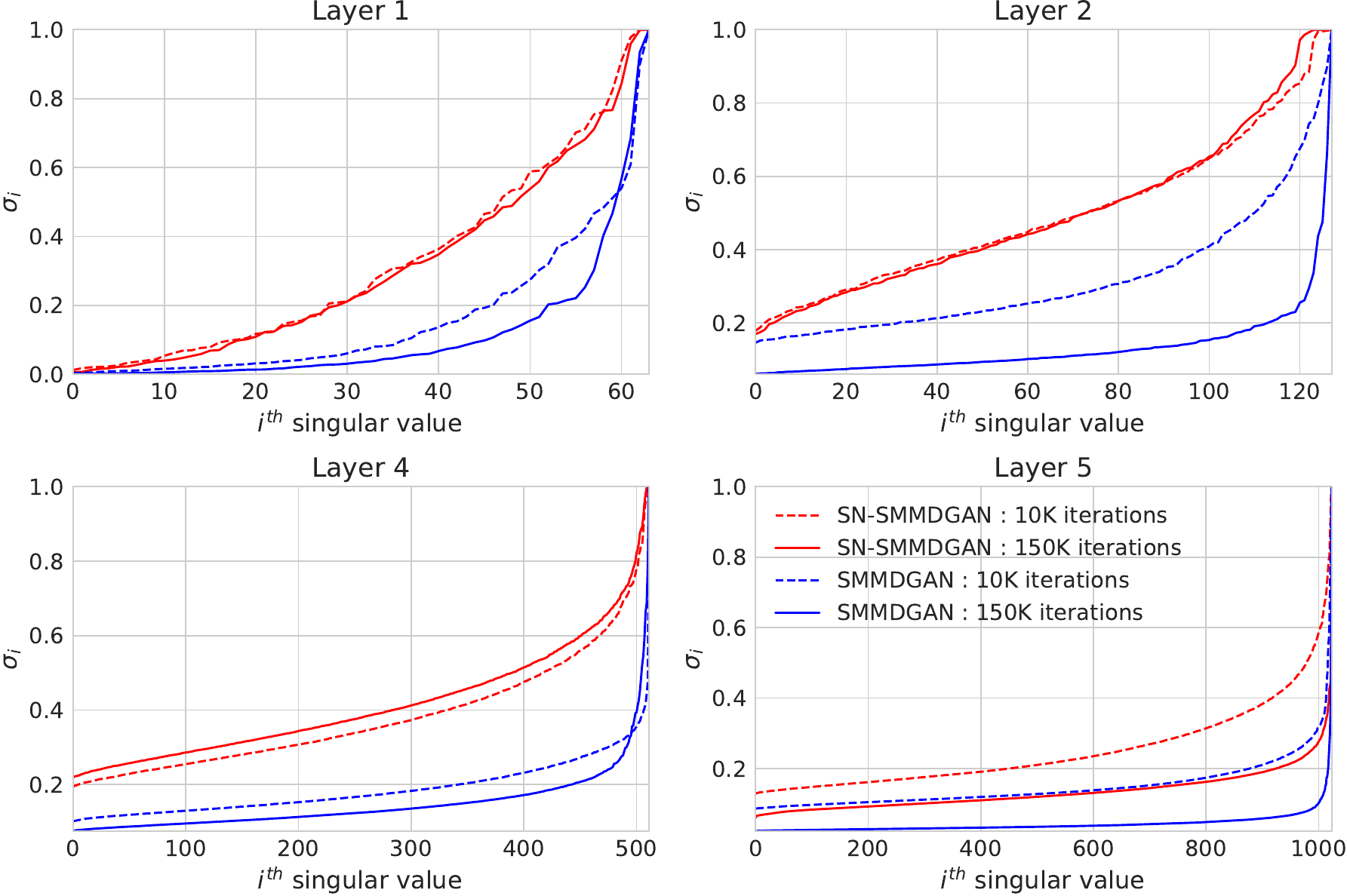}
        \caption{Singular values at different layers, for the same setup as \cref{fig:celebA_scores_and_singular_values}.}
       \label{fig:singular_values-full}
\end{figure}

\begin{figure}[ht]
	        \centering
        \includegraphics[width=\linewidth]{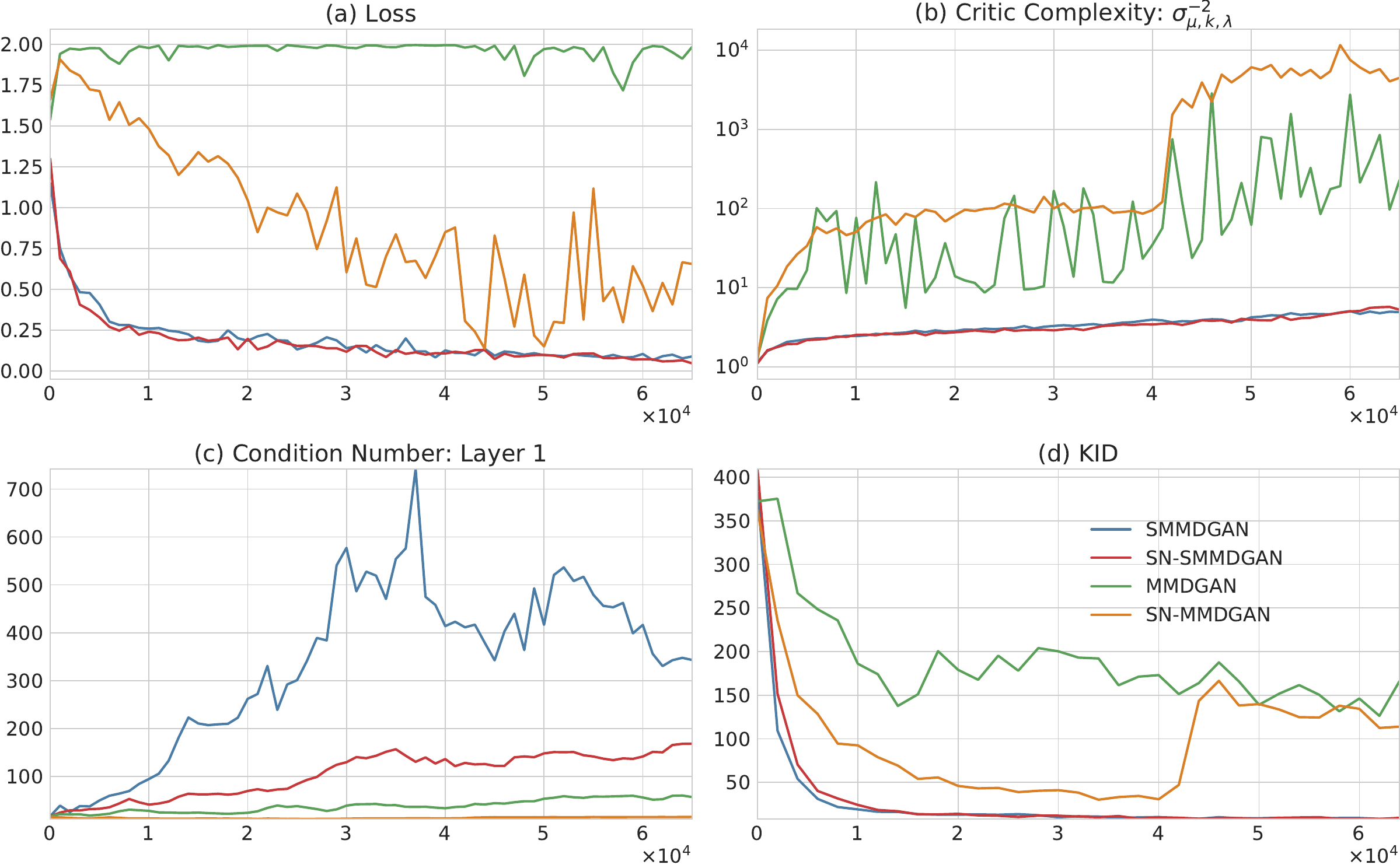}
        \caption{Evolution of various quantities per generator iteration on CelebA during training. 4 models are considered: (SMMDGAN, SN-SMMDGAN, MMDGAN, SN-MMDGAN). (a) Loss: $\SMMD^2=\sigma_{\mu, k, \lambda}^2 \MMD^2_k$ for SMMDGAN and SN-SMMDGAN, and $\MMD_k^2$ for MMDGAN and SN-MMDGAN. The loss saturates for MMDGAN (green); spectral normalization allows some improvement in loss, but training is still unstable (orange). SMMDGAN and SN-SMMDGAN both lead to stable, fast training (blue and red). (b) SMMD controls the critic complexity well, as expected (blue and red); SN has little effect on the complexity (orange). (c) Ratio of the highest singular value to the smallest for the first layer of the critic network: $\sigma_{\max} / \sigma_{\min}$. SMMD tends to increase the condition number of the weights during training (blue), while SN helps controlling it (red). (d) KID score during training: Only variants using SMMD lead to stable training in this case.}
       \label{fig:loss_celebA}
\end{figure}

\begin{figure}[ht]
        \centering
        \includegraphics[width=\linewidth]{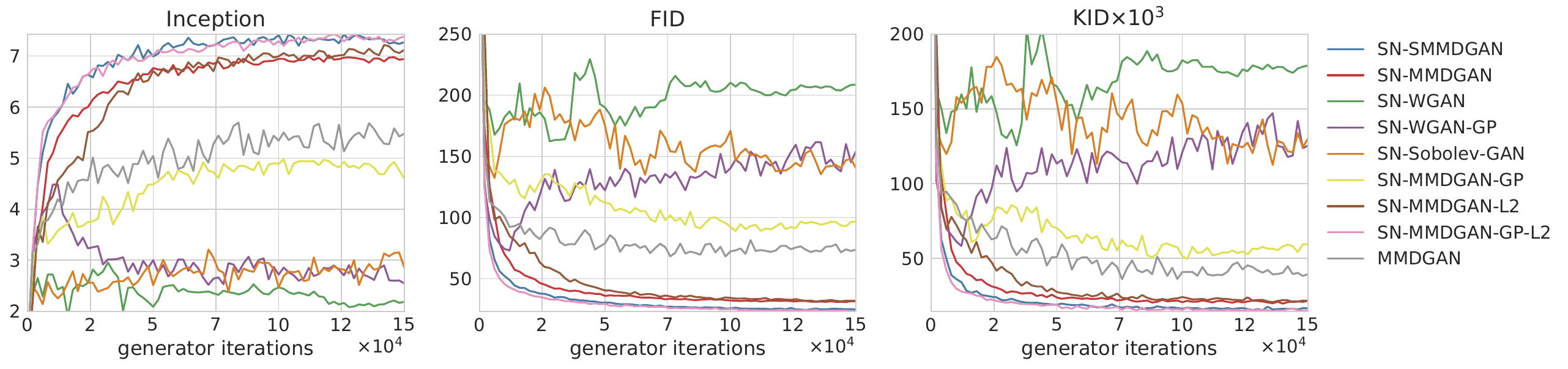}

        \caption{Evolution per iteration of different scores for variants of methods, mostly using spectral normalization, on CIFAR-10.}
       \label{fig:score_per_iter_cifar10_sn}
\end{figure}

\begin{table}[ht]
    \centering
 \caption{Mean (standard deviation) of score evaluations on CIFAR-10 for different methods using Spectral Normalization.}
    \label{tab:sn_cifar10_scores}
  \input{tables/scores_cifar10}
\end{table}

\subsection{Additional samples} \label{appendix:additional-samples}
\cref{fig:imagenet_additional,fig:celebA_samples_additional} give extra samples from the models.

\begin{figure}
        \centering
        \includegraphics[width=\linewidth]{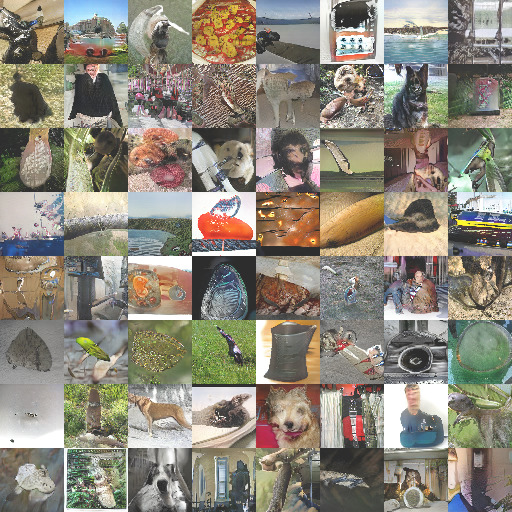}
        \caption{Samples from a generator trained on ImageNet dataset using Scaled MMD with Spectral Normalization: SN-SMMDGAN.  }
       \label{fig:imagenet_additional}
\end{figure}

\begin{figure}[ht!]
    \centering
    \begin{subfigure}[t]{0.48\textwidth}
        \centering
        \includegraphics[width=\linewidth]{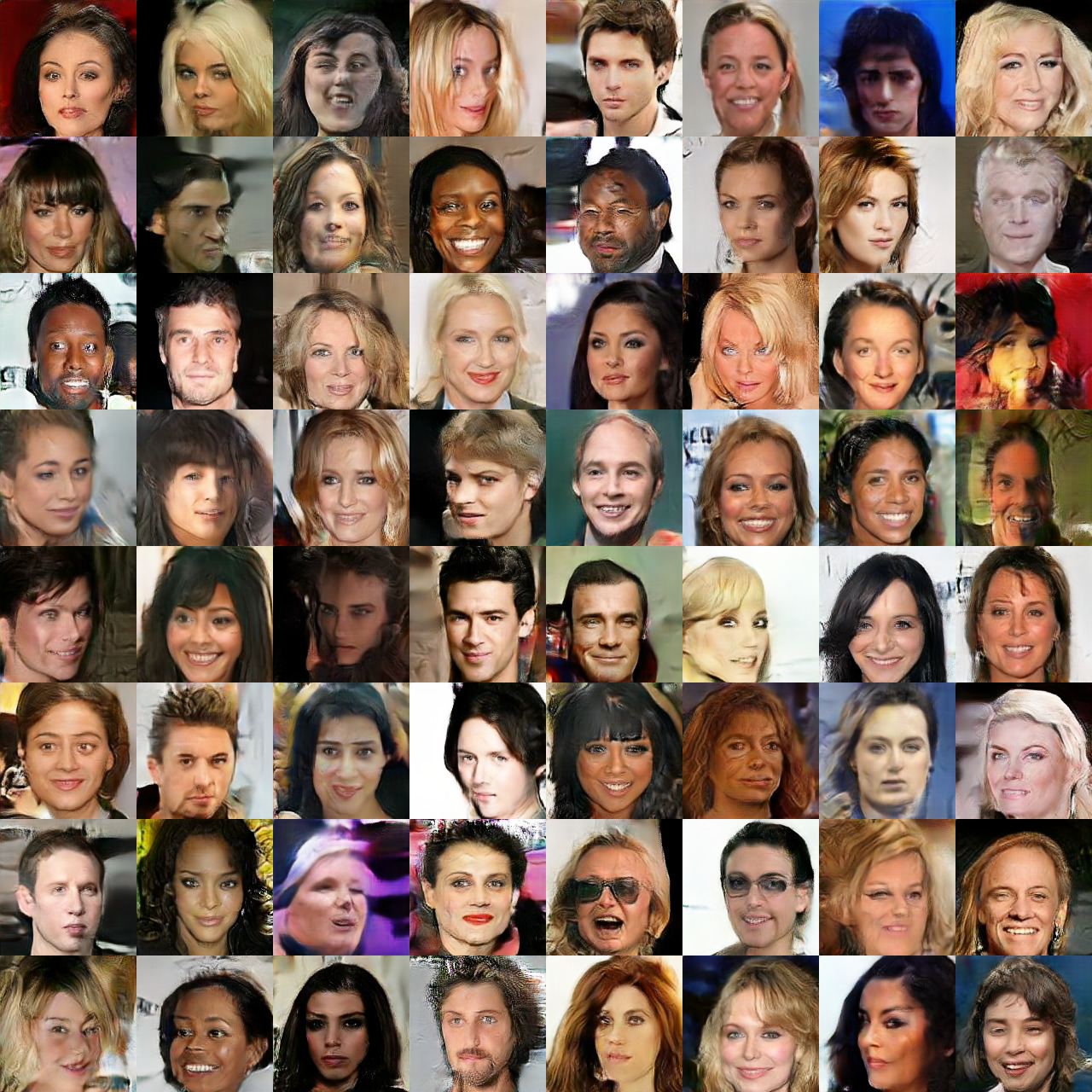}
        \caption{SNGAN} \label{fig:celebA_sngan:samples}
    \end{subfigure}
    ~
    \begin{subfigure}[t]{0.48\textwidth}
        \centering
        \includegraphics[width=\linewidth]{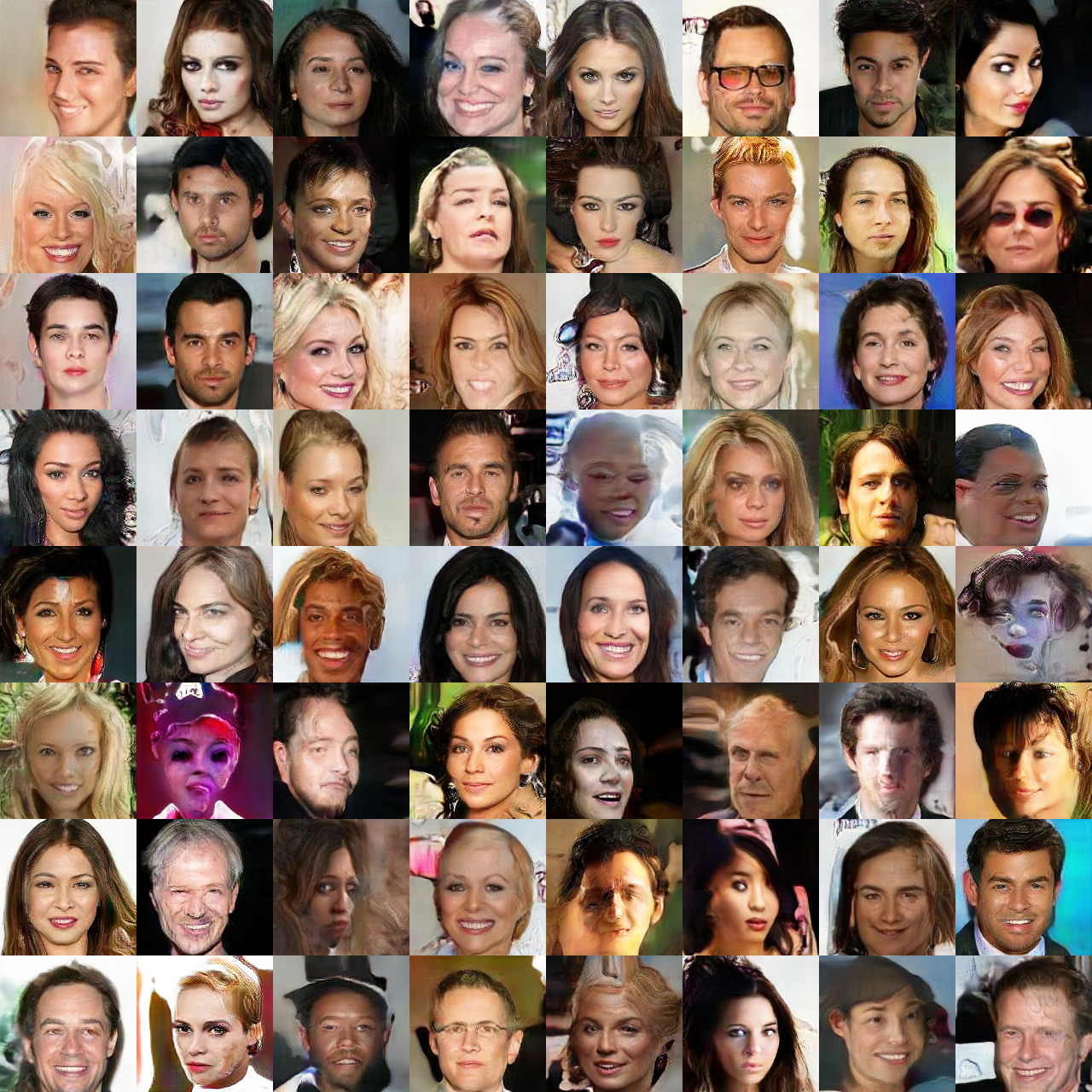}
        \caption{SobolevGAN} \label{fig:celebA_sobolev_gan:samples}
    \end{subfigure}
   \vspace{0cm}
    \begin{subfigure}[t]{0.48\textwidth}
        \centering
        \includegraphics[width=\linewidth]{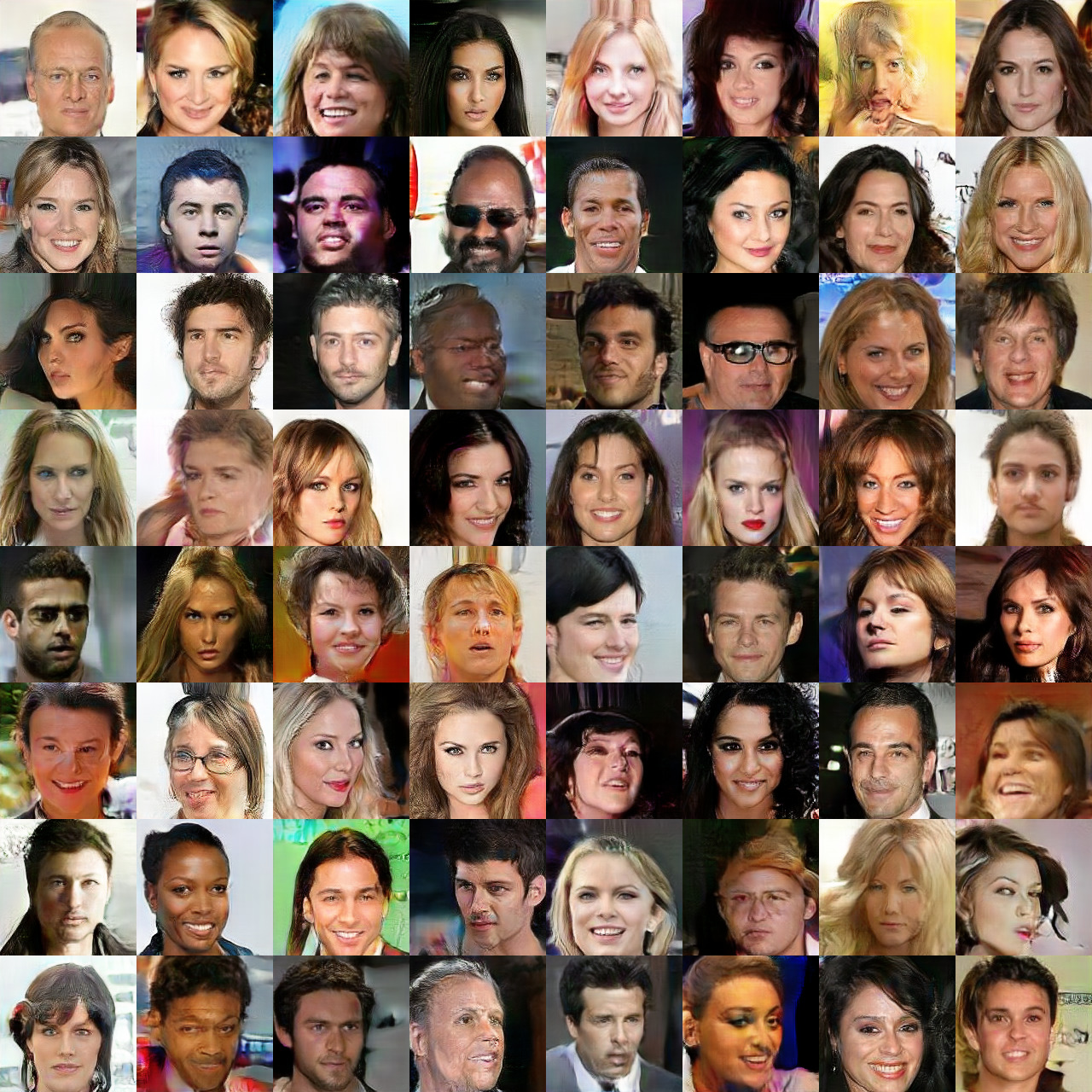}
        \caption{MMDGAN-GP-L2} \label{fig:celebA_mmd_gp_l:samples}
    \end{subfigure}
    ~
    \begin{subfigure}[t]{0.48\textwidth}
        \centering
        \includegraphics[width=\linewidth]{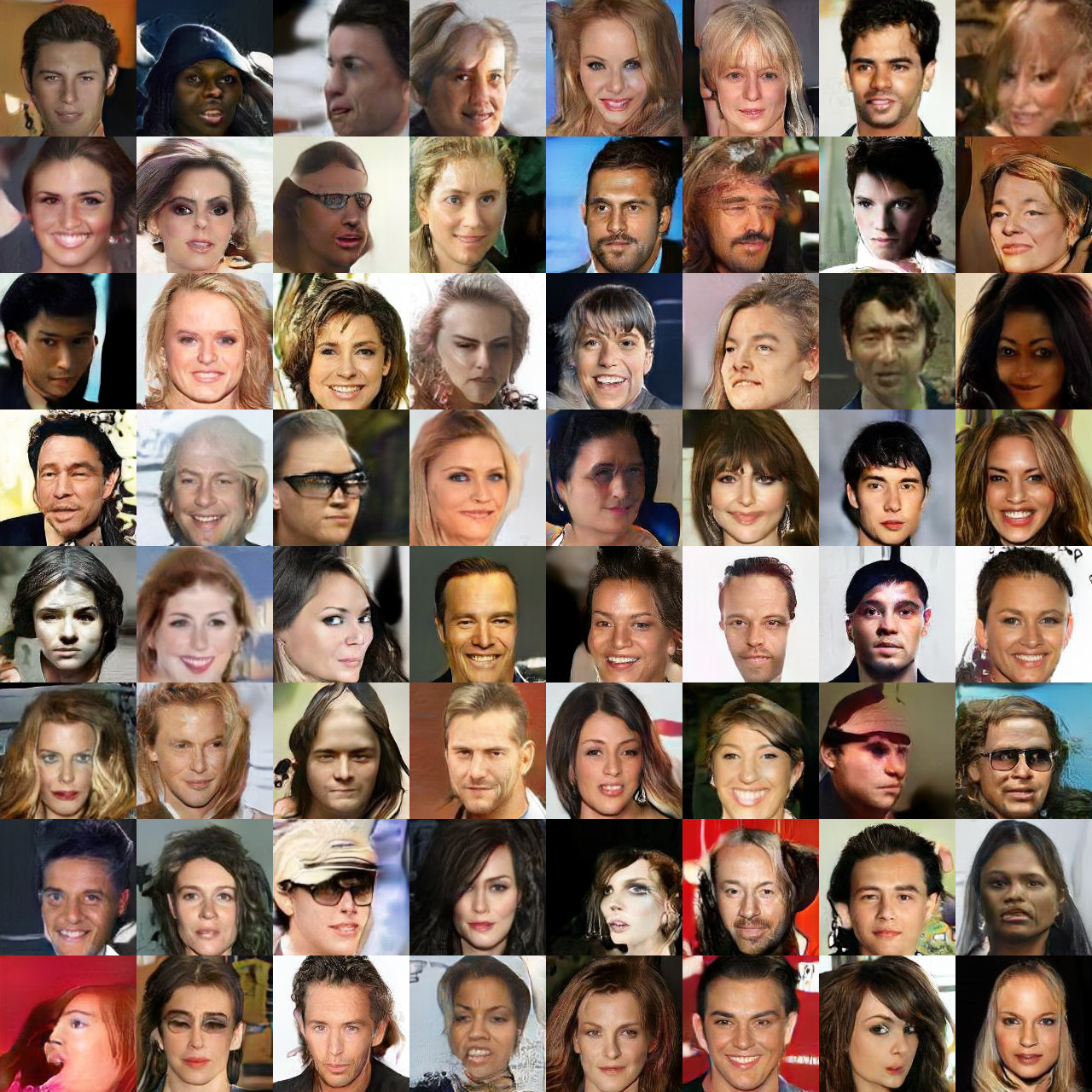}
        \caption{SN-SMMD GAN} \label{fig:celebA_sn_smmd:samples}
    \end{subfigure}
     \vspace{0cm}
    \begin{subfigure}[t]{0.48\textwidth}
        \centering
        \includegraphics[width=\linewidth]{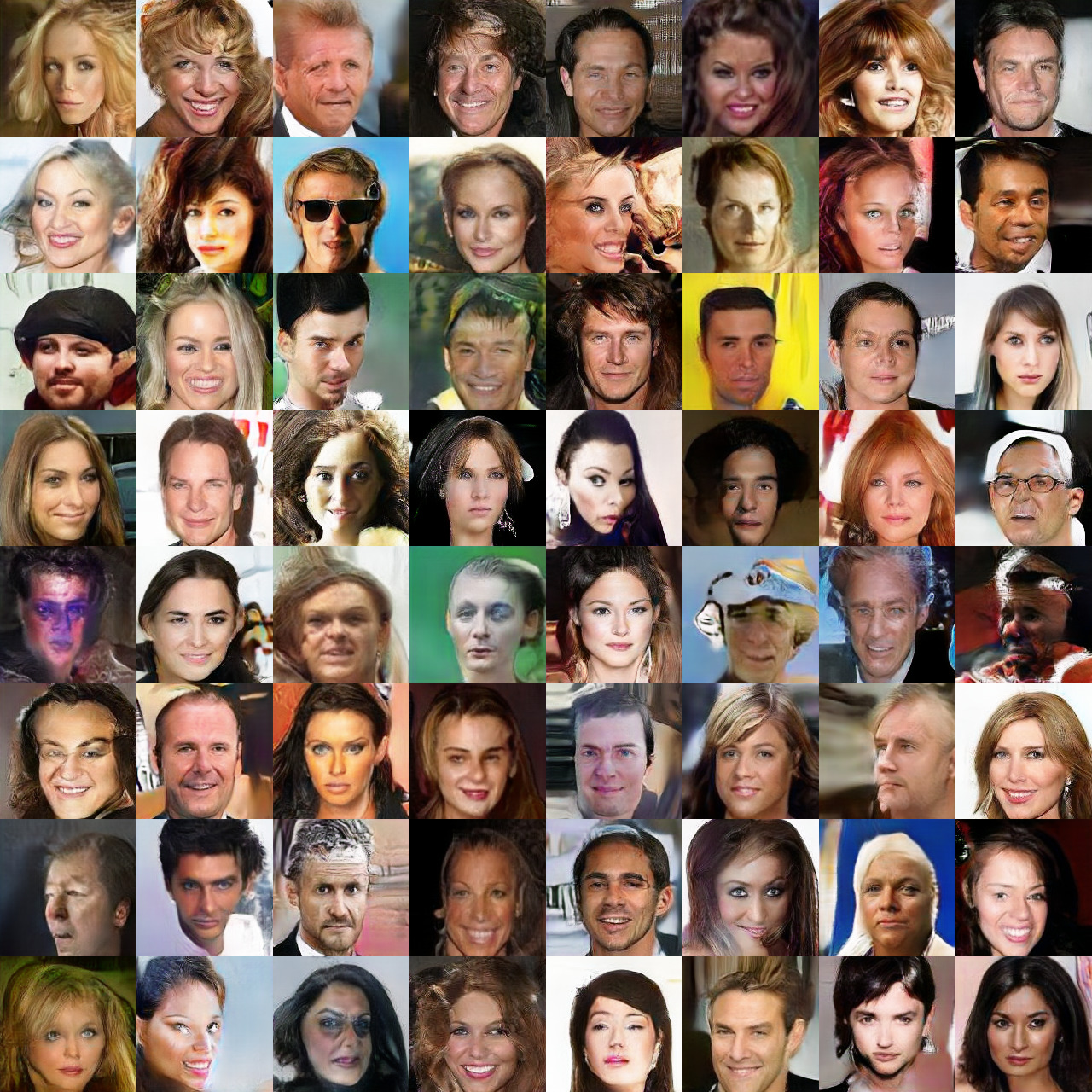}
        \caption{SN SWGAN} \label{fig:celebA_sn_swgan:samples}
    \end{subfigure}
    ~
    \begin{subfigure}[t]{0.48\textwidth}
        \centering
        \includegraphics[width=\linewidth]{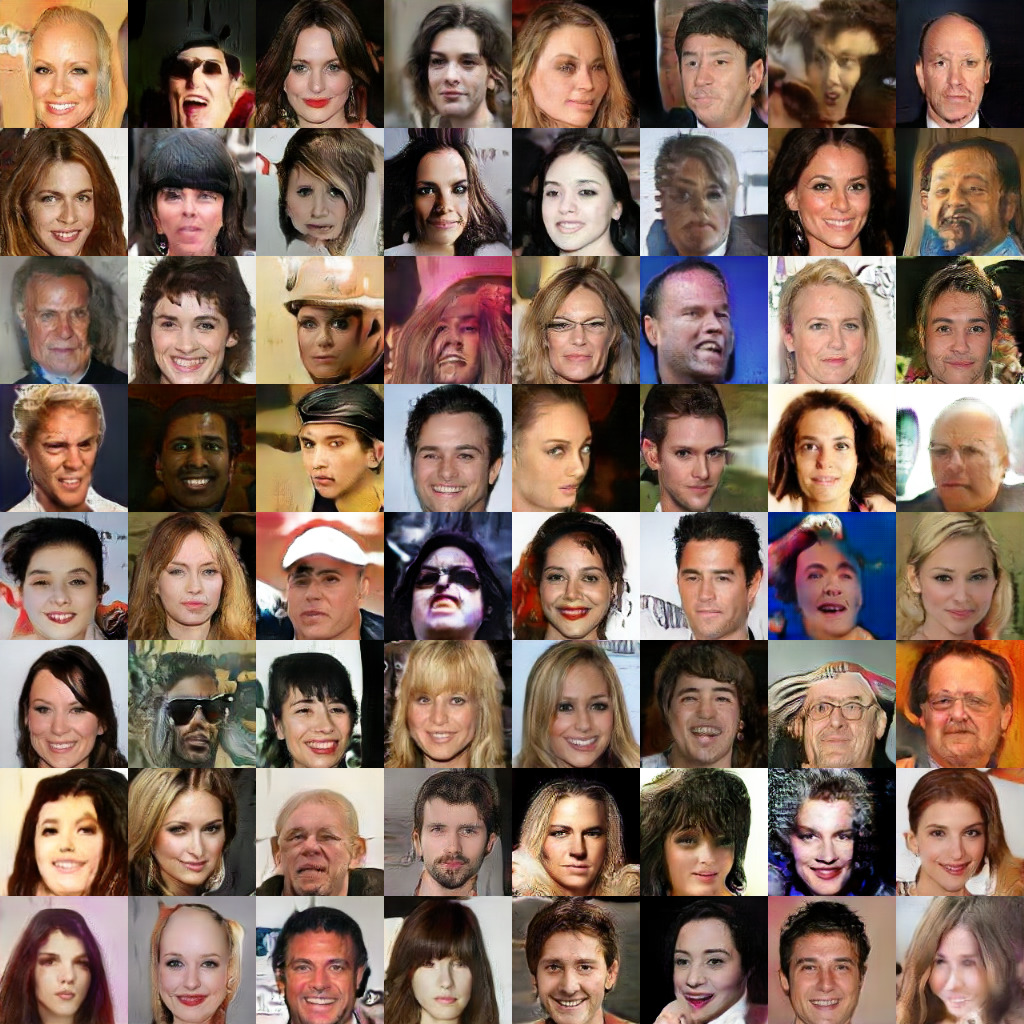}
        \caption{SMMD GAN} \label{fig:celebA_smmd:samples}
    \end{subfigure}
    \caption{Comparison of samples from different models trained on CelebA with $160\times 160$ resolution.}
    \label{fig:celebA_samples_additional}
\end{figure}

\end{document}

%% file: tables/scores_celebA_cifar10.tex
\begin{tabular}{lllllll}
\toprule
{} & \multicolumn{3}{l}{CIFAR-10} & \multicolumn{3}{l}{CelebA} \\
Method &                                    IS &                                    FID &                       KID$\times 10^3$ &                                    IS &                                    FID &                      KID$\times 10^3$ \\
\midrule
WGAN-GP      &             6.9$\mathsmaller{\pm}$0.2 &             31.1$\mathsmaller{\pm}$0.2 &             22.2$\mathsmaller{\pm}$1.1 &             2.7$\mathsmaller{\pm}$0.0 &             29.2$\mathsmaller{\pm}$0.2 &            22.0$\mathsmaller{\pm}$1.0 \\
MMDGAN-GP-L2 &             6.9$\mathsmaller{\pm}$0.1 &             31.4$\mathsmaller{\pm}$0.3 &             23.3$\mathsmaller{\pm}$1.1 &             2.6$\mathsmaller{\pm}$0.0 &             20.5$\mathsmaller{\pm}$0.2 &            13.0$\mathsmaller{\pm}$1.0 \\
Sobolev-GAN  &             7.0$\mathsmaller{\pm}$0.1 &             30.3$\mathsmaller{\pm}$0.3 &             22.3$\mathsmaller{\pm}$1.2 &     \textbf{2.9$\mathsmaller{\pm}$0.0}&             16.4$\mathsmaller{\pm}$0.1 &            10.6$\mathsmaller{\pm}$0.5 \\
SMMDGAN      &             7.0$\mathsmaller{\pm}$0.1 &             31.5$\mathsmaller{\pm}$0.4 &             22.2$\mathsmaller{\pm}$1.1 &             2.7$\mathsmaller{\pm}$0.0 &             18.4$\mathsmaller{\pm}$0.2 &            11.5$\mathsmaller{\pm}$0.8 \\
SN-GAN       &     \textbf{7.2$\mathsmaller{\pm}$0.1}&             26.7$\mathsmaller{\pm}$0.2 &    \textbf{16.1$\mathsmaller{\pm}$0.9} &             2.7$\mathsmaller{\pm}$0.0 &             22.6$\mathsmaller{\pm}$0.1 &            14.6$\mathsmaller{\pm}$1.1 \\
SN-SWGAN     &     \textbf{7.2$\mathsmaller{\pm}$0.1}&             28.5$\mathsmaller{\pm}$0.2 &    \textbf{17.6$\mathsmaller{\pm}$1.1} &             2.8$\mathsmaller{\pm}$0.0 &             14.1$\mathsmaller{\pm}$0.2 &            \phantom{0}7.7$\mathsmaller{\pm}$0.5 \\
SN-SMMDGAN   &     \textbf{7.3$\mathsmaller{\pm}$0.1}&     \textbf{25.0$\mathsmaller{\pm}$0.3}&    \textbf{16.6$\mathsmaller{\pm}$2.0} &             2.8$\mathsmaller{\pm}$0.0 &     \textbf{12.4$\mathsmaller{\pm}$0.2}&  \textbf{\phantom{0}6.1$\mathsmaller{\pm}$0.4} \\
\bottomrule
\end{tabular}

%% file: tables/scores_imagenet.tex
\begin{tabular}{llll}
\toprule
Method &                                     IS &                                    FID &                       KID$\times 10^3$ \\
\midrule
BGAN       &             10.7$\mathsmaller{\pm}$0.4 &             43.9$\mathsmaller{\pm}$0.3 &             47.0$\mathsmaller{\pm}$1.1 \\
SN-GAN     &             \textbf{11.2$\mathsmaller{\pm}$0.1} &             47.5$\mathsmaller{\pm}$0.1 &             44.4$\mathsmaller{\pm}$2.2 \\
SMMDGAN    &             10.7$\mathsmaller{\pm}$0.2 &             38.4$\mathsmaller{\pm}$0.3 &             39.3$\mathsmaller{\pm}$2.5 \\
SN-SMMDGAN & 10.9$\mathsmaller{\pm}$0.1 &  $\textbf{36.6$\mathsmaller{\pm}$0.2}$ &  $\textbf{34.6$\mathsmaller{\pm}$1.6}$ \\
\bottomrule
\end{tabular}

%% file: tables/scores_cifar10.tex
\begin{tabular}{llll}
\toprule
{} &                                    IS &                                    FID &                       KID$\times 10^3$ \\
Method          &                                       &                                        &                                        \\
\midrule
MMDGAN          &             5.5$\mathsmaller{\pm}$0.0 &             73.9$\mathsmaller{\pm}$0.1 &             39.4$\mathsmaller{\pm}$1.5 \\
SN-WGAN         &             2.2$\mathsmaller{\pm}$0.0 &            208.5$\mathsmaller{\pm}$0.2 &            178.9$\mathsmaller{\pm}$1.5 \\
SN-WGAN-GP      &             2.5$\mathsmaller{\pm}$0.0 &            154.3$\mathsmaller{\pm}$0.2 &            125.3$\mathsmaller{\pm}$0.9 \\
SN-Sobolev-GAN  &             2.9$\mathsmaller{\pm}$0.0 &            140.2$\mathsmaller{\pm}$0.2 &            130.0$\mathsmaller{\pm}$1.9 \\
SN-MMDGAN-GP    &             4.6$\mathsmaller{\pm}$0.1 &             96.8$\mathsmaller{\pm}$0.4 &             59.5$\mathsmaller{\pm}$1.4 \\
SN-MMDGAN-L2    &             7.1$\mathsmaller{\pm}$0.1 &             31.9$\mathsmaller{\pm}$0.2 &             21.7$\mathsmaller{\pm}$0.9 \\
SN-MMDGAN       &             6.9$\mathsmaller{\pm}$0.1 &             31.5$\mathsmaller{\pm}$0.2 &             21.7$\mathsmaller{\pm}$1.0 \\
SN-MMDGAN-GP-L2 &             6.9$\mathsmaller{\pm}$0.2 &             32.3$\mathsmaller{\pm}$0.3 &             20.9$\mathsmaller{\pm}$1.1 \\
SN-SMMDGAN      &  $\textbf{7.3$\mathsmaller{\pm}$0.1}$ &  $\textbf{25.0$\mathsmaller{\pm}$0.3}$ &  $\textbf{16.6$\mathsmaller{\pm}$2.0}$ \\
\bottomrule
\end{tabular}